%% file: arXiv.tex
\documentclass[preprint]{article}

\usepackage{microtype}
\usepackage{graphicx}
\usepackage{subcaption}
\usepackage{booktabs} 

\usepackage{hyperref}

\usepackage{icml2026}

\input{preamble}

\usepackage{amsmath}
\usepackage{amssymb}
\usepackage{mathtools}
\usepackage{amsthm}

\definecolor{gray}{rgb}{0.5,0.5,0.5}
\definecolor{pygreen}{rgb}{0.0, 0.5, 0.0}
\definecolor{pyred}{rgb}{0.7, 0.0, 0.0}
\definecolor{pyblue}{rgb}{0.0, 0.0, 0.7}
\definecolor{pygray}{rgb}{0.5, 0.5, 0.5}
\definecolor{pydarkgray}{rgb}{0.3, 0.3, 0.3}

\usepackage[capitalize,noabbrev]{cleveref}

\usepackage[textsize=tiny]{todonotes}

\icmltitlerunning{Fast Model Selection and Stable Optimization for SGMLMoE Models}

\begin{document}
\include{math_commands}
\twocolumn[
  \icmltitle{Fast Model Selection and Stable Optimization for Softmax-Gated Multinomial-Logistic Mixture of Experts Models}

  .
  \icmlsetsymbol{equal}{*}

  \begin{icmlauthorlist}
    \icmlauthor{TrungKhang Tran}{equal,1}
    \icmlauthor{TrungTin Nguyen}{equal,2,3}
    \icmlauthor{Md Abul Bashar}{2,3}
    \icmlauthor{Nhat Ho}{4}
    \icmlauthor{Richi Nayak}{2,3}
    \icmlauthor{Christopher Drovandi}{2,3}
  \end{icmlauthorlist}

  \icmlaffiliation{1}{School of Computing, National University of Singapore, Singapore.}
  \icmlaffiliation{2}{ARC Centre of Excellence for the Mathematical Analysis of Cellular Systems.}
  \icmlaffiliation{3}{School of Mathematical Sciences, Queensland University of Technology, Brisbane, Australia.}
    \icmlaffiliation{4}{Department of Statistics and Data Science, University of Texas at Austin, Austin, USA}

  \icmlcorrespondingauthor{TrungTin Nguyen}{t600.nguyen@qut.edu.au}

  \icmlkeywords{Machine Learning, ICML}

  \vskip 0.3in
]

\printAffiliationsAndNotice{\textsuperscript{*}Co-first Author.}

\begin{abstract}
Mixture-of-Experts (MoE) architectures combine specialized predictors through a learned gate and are effective across regression and classification, but for classification with softmax multinomial-logistic gating, rigorous guarantees for stable maximum-likelihood training and principled model selection remain limited. We address both issues in the full-data (batch) regime. First, we derive a batch minorization-maximization (MM) algorithm for softmax-gated multinomial-logistic MoE using an explicit quadratic minorizer, yielding coordinate-wise closed-form updates that guarantee monotone ascent of the objective and global convergence to a stationary point (in the standard MM sense), avoiding approximate M-steps common in EM-type implementations. Second, we prove finite-sample rates for conditional density estimation and parameter recovery, and we adapt dendrograms of mixing measures to the classification setting to obtain a sweep-free selector of the number of experts that achieves near-parametric optimal rates after merging redundant fitted atoms. Experiments on biological protein--protein interaction prediction validate the full pipeline, delivering improved accuracy and better-calibrated probabilities than strong statistical and machine-learning baselines.
\end{abstract}

\section{Introduction}

\subsection{Mixture of Experts Models}\label{sec:intro:moe}
Mixture of experts (MoE) models \citep{jacobs_adaptive_1991,jordan_hierarchical_1994} represent heterogeneous predictor--response relationships by combining a \emph{gating} model with multiple \emph{expert} models, both depending on the input. The gate assigns input-dependent weights to experts, enabling conditional computation \citep{bengio_deep_2013,chen_towards_2022} and inducing a data-adaptive partition of the input domain. In classification, softmax-gated multinomial-logistic MoE (SGMLMoE) \citep{chen_improved_1999,yuksel_variational_2010,huynh2019estimation,pham_functional_2022,nguyen_general_2024} offer calibrated multiclass probabilities while retaining the conditional-computation benefits of gating. MoE theory is supported by approximation and estimation guarantees in several regimes, including classical mixtures \citep{genovese_rates_2000,rakhlin_risk_2005,nguyen_convergence_2013,ho_convergence_2016,ho_strong_2016,nguyen_approximation_2020,nguyen_approximation_2022,chong_risk_2024}, mixtures of regressions \citep{do_strong_2025,ho_convergence_2022}, and broader MoE frameworks \citep{jiang_hierarchical_1999,nguyen_bayesian_2024,norets_approximation_2010,nguyen_universal_2016,nguyen_approximation_2019,nguyen_approximations_2021,nguyen_demystifying_2023,nguyen_general_2024,nguyen_towards_2024}; see also surveys \citep{yuksel_twenty_2012,masoudnia_mixture_2014,nguyen_practical_2018,nguyen_model_2021,cai_survey_2025}. Despite this progress, for SGMLMoE two core issues remain underdeveloped: \emph{stable} maximum-likelihood training with provable convergence, and \emph{model selection} (especially the number of experts) with rigorous guarantees and computational efficiency.

\subsection{Stable Maximum-likelihood Training via MM in the Full-data Regime}\label{sec_intro_onlineMM}
Likelihood-based estimation for SGMLMoE is challenging because the observed-data objective couples softmax gating and multinomial-logistic experts in a highly non-convex way. EM \citep{dempster1977maximum,meilijson_fast_1989,mclachlan_em_1997} is the classical template for latent-variable models, but in multinomial-logistic MoE the M-steps typically lack closed forms, so practical EM relies on inner-loop or approximate maximization, which can complicate monotonicity and global convergence guarantees. MM algorithms \citep{ortega_iterative_1970,leeuw_application_1977,de_leeuw_convergence_1977,lange2016mm,sun_majorization_minimization_2017,nguyen_introduction_2017,lange_nonconvex_2021,mairal_incremental_2015,mairal2013stochastic} instead optimize explicit surrogates that majorize/minorize the objective and can yield simple updates with stability guarantees. We focus on the \emph{batch} regime, avoiding the stochastic-approximation noise inherent in incremental and mini-batch variants of EM/MM \citep{borkar_stochastic_2008,kushner_stochastic_2003,dieuleveut_stochastic_2023,fort_stochastic_2023,wang_spiderboost_2019,cappe_Online_2009,cappe_online_2011,fort_fast_2021,fort_stochastic_2020,nguyen_mini_batch_2020,kuhn_properties_2020,karimi_global_2019,corff_online_2013,karimi_non_asymptotic_2019,oudoumanessah2024,oudoumanessah2025,chen_stochastic_2018,fort_geom_spider_em_2021,fort_perturbed_2021} and related variational/free-energy views \citep{neal_view_1998}.

\subsection{Model Selection in SGMLMoE Models}\label{sec:intro:related:modelselection}
Selecting the number of experts remains central. Classical criteria require fitting multiple model sizes and then applying AIC \citep{akaike_new_1974,fruhwirth_schnatter_analysing_2018}, BIC \citep{schwarz_estimating_1978,khalili_estimation_2024,forbes_mixture_2022,berrettini_identifying_2024,forbes_summary_2022,nguyen_modifications_2025,ho_unified_2025}, ICL \citep{biernacki_assessing_2000,fruhwirth_schnatter_labor_2012}, eBIC \citep{foygel_extended_2010,nguyen_joint_2024}, or SWIC \citep{sin_information_1996,westerhout_asymptotic_2024}. For MoE, over-specification can induce non-identifiability and singularities, undermining the usual asymptotic justifications. Alternatives include non-asymptotic penalization \citep{nguyen_non_asymptotic_2021,nguyen_model_2022,nguyen_non_asymptotic_2022,nguyen_non_asymptotic_2023,montuelle_mixture_2014,nguyen_non_asymptotic_Lasso_2023} and Bayesian/post-processing strategies such as merge-truncate-merge \citep{fruhwirth_schnatter_keeping_2019,zens_bayesian_2019,guha_posterior_2021,nguyen_bayesian_2024_JNPS}, but many remain tied to multi-$K$ sweeps. A recent sweep-free direction uses dendrograms of mixing measures \citep{do_dendrogram_2024,thai_model_2025,hai_dendrograms_2026}, developed mainly for Gaussian-expert settings; we adapt this viewpoint to multinomial-logistic classification and integrate it with stable batch-MM training.

\subsection{Contributions and Paper Organization}\label{sec_intro_overall_contribution}
We develop a unified framework for SGMLMoE that couples \emph{stable batch maximum-likelihood training} with \emph{sweep-free model selection}. On the optimization side, we construct an explicit quadratic MM surrogate, leading to coordinate-wise closed-form updates that exactly minimize the surrogate each iteration and therefore guarantee monotone ascent and global convergence to a stationary point without inner-loop M-steps. On the statistical side, we establish finite-sample rates for conditional density estimation and parameter recovery, and we introduce a dendrogram-based aggregation path on mixing measures that restores near-parametric rates after merging redundant fitted atoms (see \cref{tab_rates_sgmlmoe}) and yields a consistent order selector without multi-$K$ training (see \cref{fig_Mergeproc}).  

\textbf{Organization.} \cref{sec_SMLMoE} introduces SGMLMoE and the likelihood objective. \cref{section_batchMM_SGMLMoE} presents the batch MM algorithm and its monotonicity. \cref{sec_rate_aware_sgmlmoe} develops the Voronoi loss, aggregation path, convergence rates, and dendrogram selection criterion. Experiments are reported in \cref{sec_numerical_experiment}. All additional technical proofs and implementation details are deferred to the appendix.

\begin{figure}
    \centering
    \includegraphics[width=.7\linewidth]{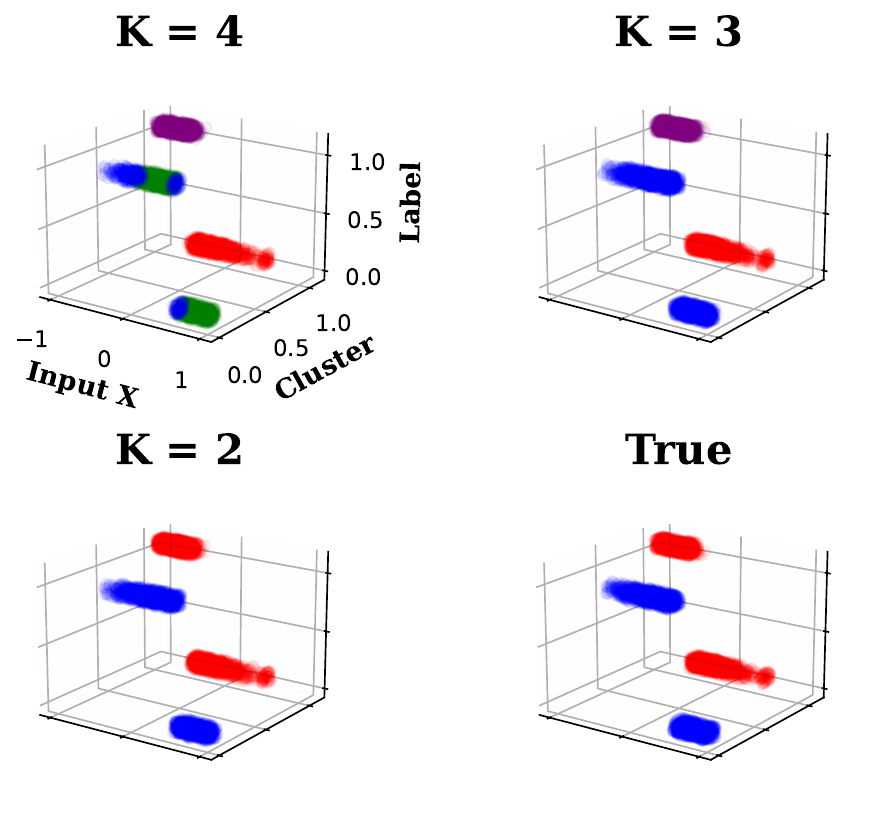}
    \caption{Illustration of the DSC merging path on an over-specified fit ($K=4, M = 2$): successive merges remove near-duplicate experts and recover the true expert structure ($K_0 = 2, M = 2$).}
    \label{fig_Mergeproc}
\end{figure}

\section{Preliminaries}\label{sec_SMLMoE}
{\bf SGMLMoE Models.} We refer to the output $\ry\in\sY$ as the response variable and the input $\mathbf{\rvx}\in\sX\subset\sR^P$, $P\in\nset$, as the predictor variable. We consider a batch dataset of size $N$, denoted by $\{(\rvx_n,\ry_n):n\in[N]\}$, consisting of i.i.d.\ copies of $(\rvx,\ry)$, where the conditional law of $\ry$ given $\rvx=\vx$ is characterized by an unknown conditional probability mass function (or density) $\pi(\cdot\mid \rvx=\vx)$. The corresponding observed values are written as $(\vx,\evy)$. Motivated by universal approximation results for MoE models, we approximate $\pi(\cdot\mid \rvx=\vx)$ by a SGMLMoE of the form
\begin{align}
\label{eq_def_SGMoE}
	s_{\vtheta}(\evy \mid \vx) = \sum_{k=1}^{K}\sfg_k\!\big(\vw(\vx)\big)\,\ex_k\!\big(\evy;\vv(\vx)\big),
\end{align}
where $\vtheta$ collects all unknown parameters and $K$ is the number of experts. The gating probabilities are defined by the softmax map $ \sfg_k\!(\vw(\vx))
=\exp\!\big(w_k(\vx)\big)/(\sum_{l=1}^{K}\exp\!\big(w_l(\vx)\big)),$ $\forall k\in[K],$
with gating score functions $\vw(\vx)=(w_1(\vx),\ldots,w_K(\vx))$. In the discrete-output setting, we take $\sY=[M]$, $M\in\sN$, and specify each expert as a multinomial-logistic model
    $\ex_k\!\big(\evy=m;\vv(\vx;\vupsilon)\big)
    =\exp\!\big(\evv_{m,k}(\vx)\big)/\sum_{l=1}^{M}\exp\!\big(\evv_{l,k}(\vx)\big),$ $ \forall k\in[K],\ \forall m\in[M],$
where the expert scores $\vv(\vx)=(\evv_{m,k}(\vx))_{m\in[M],k\in[K]}$ are modeled as polynomial functions of $\vx$. Concretely, for degrees $D_W,D_V\in\sN$,
$w_k(\vx)
    = \sum_{d=0}^{D_W}\vomega_{k,d}^{\top}\vx^d
    = \sum_{d=0}^{D_W}\Big(\sum_{p=1}^{P}\omega_{k,d,p}\evx_p^d\Big)$, 
    $\evv_{m,k}(\vx)
    = \sum_{d=0}^{D_V}\vupsilon_{m,k,d}^{\top}\vx^d
    = \sum_{d=0}^{D_V}\Big(\sum_{p=1}^{P}\upsilon_{m,k,d,p}\evx_p^d\Big)$
with coefficient collections $\vomega=(\omega_{k,d,p})_{k\in[K],\,d\in\{0,\ldots,D_W\},\,p\in[P]}$ and
$\vupsilon=(\upsilon_{m,k,d,p})_{m\in[M],\,k\in[K],\,d\in\{0,\ldots,D_V\},\,p\in[P]}$. The full parameter vector is denoted by
$\vtheta=\big(\vomega,\vupsilon\big)$. For simplify, we take $D_W = D_V = D$ and leave the general case $D_W \neq D_V$ for future work.

{\bf Maximum Log-likelihood Estimator.} We denote by $\vtheta^{0}$ a  maximizer of the population objective
$\vtheta \mapsto \Ebd{\ry\mid\rvx\sim\PP}{\log s_{\vtheta}(\ry\mid\rvx)}$.
Given the batch dataset $\{(\vx_n,\evy_n)\}_{n=1}^N$, we estimate $\vtheta$ by maximum likelihood,
$\widehat{\vtheta}_N \in \arg\max_{\vtheta\in\sT}\; \frac{1}{N}\mathcal{L}(\vtheta),$
where the observed-data log-likelihood is
\begin{align}
\label{multi_D_lossfunc}
\mathcal{L}(\vtheta)
=\sum_{n=1}^{N}\log\!\left[\sum_{k=1}^{K}\sfg_k\!\big(\vw(\vx_n)\big)\,\ex_k\!\big(\evy_n;\vv(\vx_n)\big)\right].
\end{align}
In this paper we work in the full-data (batch) regime, focusing on stable maximum-likelihood fitting for \cref{eq_def_SGMoE}--\cref{multi_D_lossfunc} without invoking stochastic-approximation devices used in incremental or mini-batch variants.

{\bf Identifiability of SGMLMoE Models.}\label{section_identifiability_SGMoE}
Because the softmax gate is invariant to additive shifts, the gating parameters are not fully identifiable. In particular, if we translate each coefficient vector as $\vomega_{k,d}\mapsto \vomega_{k,d}+\vt$ for some $\vt\in\sR^{P}$, then the resulting gating probabilities $\sfg_k(\vw(\vx))$ remain unchanged, so $\{\vomega_{k,d}\}$ are identifiable only up to a common translation. This phenomenon, and standard remedies for it, are discussed in the context of establishing convergence rates for maximum-likelihood estimation of \mdg models in \cite{nguyen_demystifying_2023}. Following \cite{hennig2000identifiablity,jiang1999identifiability}, we enforce identifiability by fixing a reference expert and imposing, without loss of generality, the constraint $\{\vomega_{K,d,p}\}_{d\in\{0,\ldots,D\},\,p\in[P]}=\zero$, which yields the equivalent parameterization $\sfg_K(\vx;\vomega)=1-\sum_{k=1}^{K-1}\sfg_k(\vx;\vomega)$ and $\sfg_k(\vx;\vomega)=\exp(w_k(\vx;\vomega_k))\big/\!\left(1+\sum_{l=1}^{K-1}\exp(w_l(\vx;\vomega_l))\right)$ for all $k\in[K-1]$. In addition to the gating identifiability constraint, we also impose an identifiability convention for the multinomial-logistic expert networks, as considered in the convergence-rate analysis for \mdm models in \cite{nguyen_general_2024}. Specifically, we fix a reference class and set, without loss of generality, $\{\upsilon_{M,k,d,p}\}_{d\in\{0,\ldots,D\},\,p\in[P]}=\zero$ for all $k\in[K]$, so that the last class probability is determined by the remaining $M-1$ logits. Equivalently, for each expert $k$, we have $\ex_k(\vv_{M,k}(\vx))=1-\sum_{m=1}^{M-1}\ex_k(\vv_{m,k}(\vx))$ and $\ex_k(\vv_{m,k}(\vx))=\exp(\evv_{m,k}(\vx))\big/\!\left(1+\sum_{l=1}^{M-1}\exp(\evv_{l,k}(\vx))\right)$ for all $m\in[M-1]$.

\section{MM Algorithm for SGMLMoE Models}\label{section_batchMM_SGMLMoE}
We now specialize the MM framework to the full-data (batch) maximum-likelihood problem for the SGMLMoE model in \cref{eq_def_SGMoE} with discrete output $\ry\in\sY=[M]$. Let $\{(\vx_n,\ry_n)\}_{n\in[N]}$ be i.i.d.\ copies of $(\rvx,\ry)$ and recall the observed-data log-likelihood $\mathcal{L}(\vtheta)$ in \cref{multi_D_lossfunc}. For each datum $n\in[N]$ and expert $k\in[K]$, define the usual posterior responsibility at the current iterate $\vtheta^{(t)}$ by
\begin{align}\label{eq_tau_kn_batch}
\tau_{n,k}^{(t)}
=\frac{\sfg_k\!\left(\vw^{(t)}(\vx_n)\right)\,\ex_k\!\left(\vv^{(t)}_{\ry_n}(\vx_n)\right)}
{\sum_{l=1}^{K}\sfg_l\!\left(\vw^{(t)}(\vx_n)\right)\,\ex_l\!\left(\vv^{(t)}_{\ry_n}(\vx_n)\right)} .
\end{align}
To simplify notation for polynomial features, let
\begin{align*}
    \hat{\vx}_n=\left[\evx_{n,1}^{0},\dots,\evx_{n,1}^{D},\dots,\evx_{n,P}^{0},\dots,\evx_{n,P}^{D}\right]^{\top}\in\sR^{P(D+1)} .
\end{align*}
Next, we use an indicator polynomial $\I(z,l)$ to encode class labels:
$ \I(z,l)=\frac{\prod_{q=1}^{M}(z-q)}{(z-l)(z-1)!(M-z)!(-1)^{M-z}}, \forall z,l\in[M],$
so that $\I(z,l)\in\{0,1\}$. This representation is convenient for coding and yields the identity $ \evv_{m,k}(\vx)=\sum_{l=1}^{M}\I(m,l)\,\evv_{l,k}(\vx),$
which allows recovering any class logit from the vector of logits.

{\bf Batch Sufficient-statistics Notation.}
For each $n\in[N]$, define the gating-design vector in $\sR^{(K-1)P(D+1)}$:
\begin{align}\label{eq_vs_n_batch}
    \vs_n^{(t)}=\big[\tau_{n,1}^{(t)}\evx_{n,1}^{0},\dots,\tau_{n,1}^{(t)}\evx_{n,1}^{D},\dots,\tau_{n,K-1}^{(t)}\evx_{n,P}^{D}\big]^{\top},
\end{align}
and define the expert-design vector $\vr_n^{(t)}$ by stacking $r_{n,d,l,p,k}^{(t)}=\I(\ry_n,l)\,\tau_{n,k}^{(t)}\,\evx_{n,p}^{d}$, $~d\in\{0,\ldots,D\},~l\in[M],~p\in[P],~k\in[K],$
into
$    \vr_n^{(t)}=
    \big[r_{n,0,1,1,1}^{(t)},\dots,r_{n,D,1,1,1}^{(t)},\dots,r_{n,D,M,1,1}^{(t)},\dots,r_{n,D,M,P,K}^{(t)}\big]^{\top}.$
We also define the class-parameter blocks for each expert $k$ by
$\vc_k=\big[\upsilon_{1,k,0,1},\dots,\upsilon_{1,k,D,1},\dots,\vupsilon_{M,k,D,1},\dots,\upsilon_{M,k,D,P}\big]^{\top}$, $\vupsilon=\vect([\vc_1,\dots,\vc_K]).$
Finally, we aggregate the batch statistics as
$    \vs^{(t)}=\sum_{n=1}^{N}\vs_n^{(t)}$, $\vr^{(t)}=\sum_{n=1}^{N}\vr_n^{(t)}.$

{\bf Quadratic Majorizers for the Gate and Experts.}
For $k\in[K-1]$, let $w_k(\vx)=\sum_{d=0}^{D}\vomega_{k,d}^{\top}\vx^{d}$ and define the gate log-sum-exp term
$g_n(\vw)=\log\!\left(1+\sum_{k=1}^{K-1}\exp\!\left(w_k(\vx_n)\right)\right).$
For experts, define for each $k\in[K]$ the multinomial log-sum-exp term
$    e_n(\vc_k)=\log\!\left(1+\sum_{m=1}^{M-1}\exp\!\left(\evv_{m,k}(\vx_n)\right)\right).$
For any $K\in\sN$, we use the quadratic majorizers based on the matrices
\begin{align}\label{eq_B_kn_batch}
    \mB_{n,K}
    &=\left(\frac{3}{4}\mI_{K-1}-\frac{\bm{1}_{K-1}\bm{1}_{K-1}^{\top}}{2(K-1)}\right)\otimes\hat{\vx}_n\hat{\vx}_n^{\top}.
\end{align}
For compactness, introduce the centered increments $\bar{\vw}_t=\vw-\vw^{(t)}$ and $\bar{\vc}_{k,t}=\vc_k-\vc_k^{(t)}$. 
\begin{theorem}[Surrogate for the batch negative log-likelihood.]\label{theorem_surrogate}
    The following bound defines a batch MM surrogate for $-\mathcal{L}(\vtheta)$, with proof deferred to \cref{sec:Surrogate_SGMLMoE}:
\begin{align}
    -\mathcal{L}(\vtheta)
    \le C^{(t)}+\mathcal{S}_1(\vtheta,\vtheta^{(t)})+\mathcal{S}_2(\vtheta,\vtheta^{(t)})
    \label{eq_surrogate_batch_SGMLMoE}
\end{align}
Here $C^{(t)}$ is an independent constant w.r.t $\vtheta$, $\mathcal{S}_1(\vtheta,\vtheta^{(t)})=\sum_{n=1}^{N}\{g_n(\vw^{(t)})+\bar{\vw}_t^{\top}\nabla g_n(\vw^{(t)})+\frac{1}{2}\bar{\vw}_t^{\top}\mB_{n,K}\bar{\vw}_t\}$, and $\mathcal{S}_1(\vtheta,\vtheta^{(t)})=\sum_{n=1}^{N}\sum_{k=1}^{K}\tau_{n,k}^{(t)}\{e_n(\vc_k^{(t)})+\bar{\vc}_{k,t}^{\top}\nabla e_n(\vc_k^{(t)})+\frac{1}{2}\bar{\vc}_{k,t}^{\top}\mB_{n,M}\bar{\vc}_{k,t}\}$.
\end{theorem}

\begin{algorithm}[t]
\caption{{\bf Batch MM algorithm for SGMLMoE}}
\label{alg_batchMM_SGMLMoE_short}
\begin{algorithmic}[1]
\Require Data $\{(\vx_n,\ry_n)\}_{n=1}^{N}$; $(K,M)$; degree $D$; init.\ $\vtheta^{(0)}=(\vw^{(0)},\vv^{(0)})$; tol.\ $\varepsilon$; max iters $T$.
\Ensure $\{\vtheta^{(t)}\}_{t\ge 0}$.

\For{$t=0,1,\dots,T-1$}
    \State \textbf{E-step.} For each $n\in[N]$, compute responsibilities $\{\tau_{n,k}^{(t)}\}_{k\in[K]}$ by \cref{eq_tau_kn_batch}.
    \State Form sufficient-statistics $\vs^{(t)}=\sum_{n}\vs_n^{(t)}$ and $\vr^{(t)}=\sum_{n}\vr_n^{(t)}$ using \cref{eq_vs_n_batch} and the definition of $\vr_n^{(t)}$ (with $\I(\ry_n,l)$).
    \State Build curvature matrices $\mB_{n,K-1}=\sum_{n}\mB_{n,K}$ and $\mB_{n,M-1}=\sum_{n}\mB_{n,M}$ with blocks $\mB_{n,K},\mB_{n,M}$ from \cref{eq_B_kn_batch}.
    \State Compute $\nabla g(\vw^{(t)})=\sum_{n}\nabla g_n(\vw^{(t)})$ and the expert log-sum-exp gradients $\{\nabla e_n(\vv^{(t)})\}_{n\in[N]}$ (definitions in \cref{eq_surrogate_batch_SGMLMoE}).
    \State \textbf{M-step.} Update $(\vw,\vv)$ by the closed-form minimizers of the surrogate $\mathcal{S}(\vtheta,\vtheta^{(t)})$ in \cref{eq_surrogate_batch_SGMLMoE}: $\vw^{(t+1)}=\vw^{(t)}+\mB_{n,K-1}^{-1}\!\Big(\vs^{(t)}-\nabla g(\vw^{(t)})\Big),$ $\vv^{(t+1)}=\arg\min_{\vv}\ \mathcal{S}\big((\vw^{(t+1)},\vv),\vtheta^{(t)}\big),$
    where the explicit block-linear solve for $\vv^{(t+1)}$ is given in Appendix~\ref{app:batchMM_SGMLMoE_details}.
    \State Set $\vtheta^{(t+1)}=(\vw^{(t+1)},\vv^{(t+1)})$.
    \State \textbf{Stop.} If $|\mathcal{L}(\vtheta^{(t+1)})-\mathcal{L}(\vtheta^{(t)})|\le \varepsilon$, break.
\EndFor
\end{algorithmic}
\end{algorithm}

These updates in \cref{alg_batchMM_SGMLMoE} minimize the surrogate \cref{eq_surrogate_batch_SGMLMoE} exactly, and therefore inherit the standard MM monotone-ascent property for the observed-data likelihood via \cref{thm_MM_monotone_SGMLMoE}, whose proof follows from the MM majorization--tangency properties of $\mathcal{S}$ and exact minimization; see \cref{proof_thm_MM_monotone_SGMLMoE} in the appendix.

\begin{theorem}[MM monotonicity for batch SGMLMoE]\label{thm_MM_monotone_SGMLMoE}
Let $\mathcal{S}(\vtheta,\vtheta^{(t)})$ be the batch MM surrogate of $-\mathcal{L}(\vtheta)$ in \cref{eq_surrogate_batch_SGMLMoE} and $\vtheta^{(t+1)}$ be the output of \cref{alg_batchMM_SGMLMoE_short}. Then, $\vtheta^{(t+1)}\in\arg\min_{\vtheta}\mathcal{S}(\vtheta,\vtheta^{(t)})$
and $\mathcal{L}(\vtheta^{(t+1)})\ge \mathcal{L}(\vtheta^{(t)})$ for all $t$.
\end{theorem}

\section{Fast-rate-aware Aggregation and Sweep-free Model Selection for SGMLMoE}
\label{sec_rate_aware_sgmlmoe}

\subsection{The ``Rate Gap'' under Overfitting}
\label{subsec_motivation_rate_gap_sgmlmoe}

In the SGMLMoE classification model \cref{eq_def_SGMoE}--\cref{multi_D_lossfunc}, likelihood-based estimation is
challenging for the same structural reasons that arise in softmax-gated Gaussian experts~\citep{hai_dendrograms_2026}, but with multinomial-logistic
experts:
(i) \emph{softmax translation invariance} makes gate parameters identifiable only up to a common shift;
(ii) \emph{gate--expert coupling} yields nontrivial interactions in local expansions of $s_{\vtheta}(\cdot\mid \vx)$;
(iii) \emph{softmax numerator--denominator coupling} creates cancellation phenomena that can make densities close even
when individual components are not.

These effects become most visible when the fitted model is \emph{over-specified} ($K>K_0$): several fitted atoms can
cluster around a single true expert. In this regime, the conditional density $s_{\widehat G_N}(\cdot\mid \vx)$ can
converge at near-parametric rates in total variation, while componentwise parameters converge more slowly along
near-nonidentifiable directions. The goal of this section is to (a) introduce a Voronoi-type loss aligned with the
gate-induced partition geometry that resolves these slow directions, and (b) leverage it to build a \emph{hierarchical
aggregation path} (a dendrogram of mixing measures) that yields a \emph{sweep-free} and \emph{consistent} selector of the
number of experts.

{\bf Practical Implication (Single-fit Workflow).}
Fit one over-specified SGMLMoE (moderate $K\ge K_0$) using the batch MM procedure in \cref{alg_batchMM_SGMLMoE},
construct its dendrogram by iterative merging, and select $\widehat K$ via a height--likelihood dendrogram selection criterion (DSC). This avoids
training multiple models across candidate sizes.

\subsection{Mixing-measure Representation for SGMLMoE}
\label{subsec_mixing_measure_sgmlmoe}

{\bf Lifted Features and Intercept/Slope Decomposition.}
With the polynomial degree $D$ fixed, define the non-constant lifted feature map
$
\phi_D(\vx):=\vect\!\big([\vx^{1},\vx^{2},\ldots,\vx^{D}]\big)\in\sR^{PD}$, $\widehat{\sX}:=\phi_D(\sX)\subset\sR^{PD}.$
(Since $\sX$ is compact and $\phi_D$ is continuous, $\widehat{\sX}$ is compact via using Heine-Borel theorem).
For each gate $k\in[K-1]$, decompose
$w_k(\vx)=\bar\omega_k+\hat\vomega_k^\top \hat\vx$, $\hat\vx:=\phi_D(\vx),$
where $\bar\omega_k:=\sum_{p=1}^P \omega_{k,0,p}\in\sR$ and
$\hat\vomega_k:=\vect([\vomega_{k,1},\ldots,\vomega_{k,D}])\in\sR^{PD}$.
For each expert $k\in[K]$ and class $m\in[M-1]$, similarly write
$\evv_{m,k}(\vx)=\bar\upsilon_{m,k}+\hat\vupsilon_{m,k}^\top \hat\vx,$
with $\bar\upsilon_{m,k}:=\sum_{p=1}^P \upsilon_{m,k,0,p}\in\sR$ and
$\hat\vupsilon_{m,k}:=\vect([\vupsilon_{m,k,1},\ldots,\vupsilon_{m,k,D}])\in\sR^{PD}$.

{\bf Identifiability Conventions (as in \cref{section_identifiability_SGMoE}).}
We fix the reference expert and the reference class: $(\bar\omega_K,\hat\vomega_K)=(0,\zero),$ $(\bar\upsilon_{M,k},\hat\vupsilon_{M,k})=(0,\zero),\ \ \forall k\in[K].$

{\bf Mixing Measures.}
Let $\sT$ denote the identifiable parameter space for one SGMLMoE atom:
$\eta=\Big(\hat\vomega,\bar\omega,\ \{(\hat\vupsilon_{m},\bar\upsilon_{m})\}_{m=1}^{M-1}\Big)\in\sT,$
and let $\mathcal{O}_K(\sT)$ be the set of finite mixing measures with at most $K$ atoms:
$
\mathcal{O}_K(\sT)=\Big\{G=\sum_{i=1}^{K}\pi_i\,\delta_{\eta_i}:\ 1\le K\le K,\ \pi_i>0,\ \eta_i\in\sT\Big\}.$
We parameterize weights by $\pi_i=\exp(\bar\omega_i)$ (so $\bar\omega_i=\log\pi_i$).

{\bf Conditional Probability Mass Function Induced by $G$.}
Given $G=\sum_{i=1}^{K}\pi_i\delta_{\eta_i}$, define for $\hat\vx=\phi_D(\vx)$
$\sfg_i(\vx;G)
=\frac{\pi_i\exp(\hat\vomega_i^\top \hat\vx)}{\sum_{j=1}^{K}\pi_j\exp(\hat\vomega_j^\top \hat\vx)},$
$\ex_i(\evy=m\mid \vx;G)
=\frac{\exp(\bar\upsilon_{m,i}+\hat\vupsilon_{m,i}^\top \hat\vx)}{1+\sum_{\ell=1}^{M-1}\exp(\bar\upsilon_{\ell,i}+\hat\vupsilon_{\ell,i}^\top \hat\vx)},$
with $\ex_i(\evy=M\mid \vx;G)=1-\sum_{m=1}^{M-1}\ex_i(\evy=m\mid \vx;G)$, and
$s_G(\evy\mid \vx)=\sum_{i=1}^{K}\sfg_i(\vx;G)\,\ex_i(\evy\mid \vx;G).$
For any identifiable $\vtheta$, there exists $G(\vtheta)\in\mathcal{O}_K(\sT)$ such that
$s_{G(\vtheta)}\equiv s_{\vtheta}$.

{\bf MLE over at most $K$ experts.}
Let $G_0\in\mathcal{O}_{K_0}(\sT)$ be the true mixing measure ($2\le K_0\le K$). Define the (mixing-measure) MLE
$\widehat{G}_N\in\arg\max_{G\in\mathcal{O}_K(\sT)}\ \frac1N\sum_{n=1}^N \log s_G(\evy_n\mid \vx_n).$

\subsection{Voronoi Cells and Losses for SGMLMoE}
\label{subsec_voronoi_loss_sgmlmoe}

{\bf From ``Density Closeness'' to ``Parameter Closeness''.}
A key step in overfitted mixture/MoE analysis is an \emph{inverse inequality}: small discrepancy between
$s_G$ and $s_{G_0}$ should imply small discrepancy between $G$ and $G_0$ in a geometry-adapted loss. For softmax-gated
MoE, the adapted geometry is \emph{Voronoi}: fitted atoms are grouped into cells around each true atom.

{\bf Voronoi Partition.}
Write $G=\sum_{i=1}^{K}\pi_i\delta_{\eta_i}$ and $G_0=\sum_{k=1}^{K_0}\pi_k^0\delta_{\eta_k^0}$.
When $G$ is sufficiently close to $G_0$ (e.g.\ in the sense below), define a matching $i\mapsto k(i)\in[K_0]$ that assigns
each fitted atom $\eta_i$ to its nearest true atom $\eta_{k(i)}^0$ under the local metric on $\sT$. Define Voronoi cells
$\sV_k:=\{\,i\in[K]:\ k(i)=k\,\}$, $k\in[K_0].$

{\bf Integrated Total Variation Discrepancy.}
For $\vx\in\sX$, define $\mathcal{D}_{\mathrm{TV}}\!(s_G(\cdot\mid \vx),s_{G_0}(\cdot\mid \vx)) $ as $\frac12\sum_{m=1}^{M}|s_G(m\mid \vx)-s_{G_0}(m\mid \vx)|,$
and its covariate-averaged version $\mathcal{D}_{\mathrm{TV}}(s_G,s_{G_0})$ as $\E_{\rvx}[\mathcal{D}_{\mathrm{TV}}\!\big(s_G(\cdot\mid \rvx),s_{G_0}(\cdot\mid \rvx)\big)].$

{\bf SGMLMoE Voronoi Loss.}
For $i\in\sV_k$, define parameter differences $\Delta\hat\vomega_{ik}:=\hat\vomega_i-\hat\vomega_k^0,$ $\Delta\bar\upsilon_{ikm}:=\bar\upsilon_{m,i}-\bar\upsilon_{m,k}^0,$ $\Delta\hat\vupsilon_{ikm}:=\hat\vupsilon_{m,i}-\hat\vupsilon_{m,k}^0,\ \ m\in[M-1]$, and the loss
\begin{align}
\label{eq_VoronoiLoss_SGMLMoE}
&\mathcal{D}_{\mathrm{V}}(G,G_0)\:=\mathcal{D}_{\mathrm{E}}(G,G_0)
\\
&
+\sum_{\substack{k\in[K_0]:\,|\sV_k|>1\\ i\in\sV_k, m\in[M-1]}}
\pi_i[
\|\Delta\hat\vomega_{ik}\|^2
+(|\Delta\bar\upsilon_{ikm}|^2+\|\Delta\hat\vupsilon_{ikm}\|^2)],\nonumber
\end{align}
$\text{where }\mathcal{D}_{\mathrm{E}}(G,G_0):=\sum_{k=1}^{K_0}\Big|\sum_{i\in\sV_k}\pi_i-\pi_k^0\Big|+\sum_{\substack{k\in[K_0]:\,|\sV_k|=1\\ i\in\sV_k, m\in[M-1]}}
\pi_i[
\|\Delta\hat\vomega_{ik}\|
+(|\Delta\bar\upsilon_{ikm}|+\|\Delta\hat\vupsilon_{ikm}\|)].$

{\bf Local Structure Inequality (Inverse Bound).}
The next result is the ``density $\Rightarrow$ Voronoi'' inverse bounds developed for
SGMLMoE.
\begin{theorem}
\label{thm_VorIneq_SGMLMoE}
For $G\in\mathcal{O}_K(\sT)$, as $\mathcal{D}_{\mathrm{TV}}(s_G,s_{G_0})\to 0$,
\begin{equation}
\label{eq_LocalStructure_SGMLMoE}
\mathcal{D}_{\mathrm{TV}}(s_G,s_{G_0})\ \gtrsim\ \mathcal{D}_{\mathrm{V}}(G,G_0).
\end{equation}
\end{theorem}

\subsection{Why Merge Experts? The Dendrogram Viewpoint}
\label{subsec_merging_sgmlmoe}

{\bf The Overfitted Picture.}
When $K>K_0$, some Voronoi cell $\sV_k$ may contain multiple fitted atoms. These atoms are ``near duplicates''
from the perspective of the conditional density, and they create slow/flat directions in the likelihood geometry.
The dendrogram mechanism collapses each such cluster by repeated merging, producing a hierarchy of mixing measures
$\widehat G_N^{(K)}=\widehat G_N \ \rightsquigarrow\ \widehat G_N^{(K-1)}\ \rightsquigarrow\ \cdots\ \rightsquigarrow\
\widehat G_N^{(2)}.$
The key facts we exploit are:
(i) each merge reduces redundancy, and
(ii) the Voronoi loss becomes easier (monotonically) along the merge chain.

{\bf Atom Dissimilarity (Merge Criterion).}
For two atoms $\pi_i\delta_{\eta_i}$ and $\pi_j\delta_{\eta_j}$, define
\begin{align}
\label{eq_dissimilarity_SGMLMoE}
&\mathrm{d}\!\big(\pi_i\delta_{\eta_i},\pi_j\delta_{\eta_j}\big)
:=\frac{\pi_i\pi_j}{\pi_i+\pi_j}\big(
\|\hat\vomega_i-\hat\vomega_j\|^2\nonumber\\
&\qquad +\sum_{m=1}^{M-1}\big(\|\hat\vupsilon_{m,i}-\hat\vupsilon_{m,j}\|^2+|\bar\upsilon_{m,i}-\bar\upsilon_{m,j}|^2\big)
\big).
\end{align}

{\bf Merge Operator (Weight-averaged Atom).}
Merging atoms $i$ and $j$ produces a new atom $(\pi_*,\etab_*)$ via parameter barycenters for $\etab_*$ as follows: given any $m\in[M-1]$,
\begin{align}
\label{eq_merge_rule_SGMLMoE}
&\pi_*=\pi_i+\pi_j,~ \bar\omega_*=\log\pi_*, \hat\vupsilon_{m,*} =\frac{\pi_i}{\pi_*}\hat\vupsilon_{m,i}+\frac{\pi_j}{\pi_*}\hat\vupsilon_{m,j},\nn\\
&\bar\upsilon_{m,*} =\frac{\pi_i}{\pi_*}\bar\upsilon_{m,i}+\frac{\pi_j}{\pi_*}\bar\upsilon_{m,j},~\hat\vomega_* =\frac{\pi_i}{\pi_*}\hat\vomega_i+\frac{\pi_j}{\pi_*}\hat\vomega_j.
\end{align}

\begin{theorem}[Voronoi monotonicity along the merge chain]
\label{thm_VorChain_SGMLMoE}
Let $G^{(K)},G^{(K-1)},\ldots,G^{(K_0)}$ be obtained by repeatedly merging the closest pair of atoms according to
\cref{eq_dissimilarity_SGMLMoE}, starting from $G^{(K)}\in\mathcal{O}_K(\sT)$.
For $G^{(K)}$ sufficiently close to $G_0$, $\mathcal{D}_{\mathrm{V}}\!\big(G^{(K)},G_0\big)\ \gtrsim\
\mathcal{D}_{\mathrm{V}}\!\big(G^{(K-1)},G_0\big) \ \gtrsim\ \cdots\ \gtrsim\
\mathcal{D}_{\mathrm{V}}\!\big(G^{(K_0)},G_0\big).$
\end{theorem}

\subsection{Finite-sample Rates along the Dendrogram}
\label{subsec_rates_heights_sgmlmoe}

{\bf Heights (structural signal).}
At level $\kappa$, define the dendrogram height
$h_N^{(\kappa)}:=\min_{i\neq j\ \text{atoms of }\widehat G_N^{(\kappa)}}\ 
\mathrm{d}\!\big(\pi_i\delta_{\eta_i},\pi_j\delta_{\eta_j}\big),$ $\kappa\in[K],$
and let $h_0^{(\kappa')}$ be the analogous height on the true chain $G_0^{(\kappa')}$, $\kappa'\in[K_0]$.

{\bf Density Rate for the MLE.}
We use the known near-parametric bound (up to logs $\widetilde{\cO}\!\big((\log N/N)^{1/2}\big)$) for the MLE in total variation as established for the multinomial
classification regime in \cite{nguyen_general_2024}:
\begin{fact}[Density convergence rate]\label{lem_JointDensity_SGMLMoE}
There exist universal constants $C_1,C_2>0$ such that $\PP\!(\mathcal{D}_{\mathrm{TV}}(s_{\widehat{G}_N},s_{G_0})>C_1\sqrt{\log N/N})\lesssim N^{-C_2}.$
\end{fact}

{\bf From Density Rates to Voronoi and Height Rates.}
Combining \cref{thm_VorIneq_SGMLMoE} with \cref{lem_JointDensity_SGMLMoE} yields Voronoi control at level $K$,
and the monotonicity in \cref{thm_VorChain_SGMLMoE} propagates it down the path:
\begin{theorem}[Voronoi and height rates along the path]\label{thm_ConvergenceRate_height_Voronoi}
With probability at least $1-c_1N^{-c_2}$, for each $\kappa\in\{K_0+1,\ldots,K\}$ and $\kappa'\in[K_0]$,
$\max\{\mathcal{D}_{\mathrm{V}}(\widehat{G}_N^{(\kappa)},G_0),h_N^{(\kappa)},|h_N^{(\kappa')}-h_0^{(\kappa')}|\}\ \lesssim\ \sqrt{\log N/N}.$
\end{theorem}

\begin{table*}
\caption{Summary of (log-adjusted) density and parameter rates for SGMLMoE with standard softmax gating. 
}

\label{tab_rates_sgmlmoe}
\centering
\small
\setlength{\tabcolsep}{6pt}
\renewcommand{\arraystretch}{1.15}
\begin{tabular}{|c|c|c|c|}
\hline
\textbf{Setting} & \textbf{Density} &
\textbf{Exact-specified Parameters ($|\sV_k|=1$)} &
\textbf{Over-specified Parameters ($|\sV_k|>1$)} \\
\hline
Exact-fit
& $\widetilde{\cO}\!\big((\log N/N)^{1/2}\big)$
& $\widetilde{\cO}\!\big((\log N/N)^{1/2}\big)$
& --- \\
\hline
Over-fit
& $\widetilde{\cO}\!\big((\log N/N)^{1/2}\big)$
& $\widetilde{\cO}\!\big((\log N/N)^{1/2}\big)$
& $\widetilde{\cO}\!\big((\log N/N)^{1/4}\big)$ \\
\hline
\textbf{Merged}
& $\widetilde{\cO}\!\big((\log N/N)^{1/2}\big)$
& $\widetilde{\cO}\!\big((\log N/N)^{1/2}\big)$
& $\widetilde{\cO}\!\big((\log N/N)^{1/2}\big)$ \\
\hline
\end{tabular}
\end{table*}

\subsection{Likelihood along the Path and Sweep-free Selection}
\label{subsec_dsc_sgmlmoe}

{\bf Empirical and Population Log-likelihoods.}
For $G\in\mathcal{O}_K(\sT)$, define $\mathcal{L}(G):=\E_{(\rvx,\ry)\sim G_0}\big[\log s_G(\ry\mid \rvx)\big],$ and $\bar{\ell}_N(G):=\frac{1}{N}\sum_{n=1}^N \log s_G(\evy_n\mid \vx_n).$
Write $\bar{\ell}_N^{(\kappa)}:=\bar{\ell}_N(\widehat G_N^{(\kappa)})$.

{\bf Condition K (local lower bound).}
There exist constants $c_\alpha,c_\beta>0$ such that for all sufficiently small $\epsilon>0$ and all
$\vtheta,\vtheta^0\in\sT$ with $\|\vtheta-\vtheta^0\|\le \epsilon$, for all $(\rvx,\ry)\in\sX\times\sY$,
$\log s_{\vtheta}(\ry\mid \rvx)
\ge
(1+c_\beta\epsilon)\,\log s_{\vtheta^0}(\ry\mid \rvx)
-c_\alpha\epsilon.$

\begin{theorem}[Likelihood control along the dendrogram]\label{thm_ModelSelection_SGMLMoE}
Assume \cref{thm_ConvergenceRate_height_Voronoi} and Condition~K hold.
Then, for any $\kappa\in\{K_0+1,\ldots,K\}$,
$\bar{\ell}_N\big(\widehat{G}_N^{(\kappa)}\big)-\mathcal{L}(G_0)\ \lesssim\ (\log N/N)^{1/4}.$
Moreover, for any $\kappa'\in[K_0]$,
$\bar{\ell}_N(\widehat{G}_N^{(\kappa')})\to \mathcal{L}(G_0^{(\kappa')})$ in $\PP_{G_0}$-probability as $N\to\infty$.
\end{theorem}

{\bf DSC (height + likelihood).}
Given $1\ll \omega_N \ll (N/\log N)^{1/4}$, define the dendrogram selection criterion
$\mathrm{DSC}_N^{(\kappa)}:=-\Big(h_N^{(\kappa)}+\omega_N\,\bar{\ell}_N^{(\kappa)}\Big)$.
A practical choice is $\omega_N=\log N$. Select
$\widehat{K}_N:=\arg\min_{\kappa\in\{2,\ldots,K\}}\mathrm{DSC}_N^{(\kappa)}.$

\begin{theorem}[Consistency of DSC for SGMLMoE]\label{thm:ConvergeMerge_SGMLMoE}
Assume $K_0\ge 2$ and the conditions of \cref{thm_ModelSelection_SGMLMoE} hold.
Then $\widehat{K}_N\to K_0$ in $\PP_{G_0}$-probability as $N\to\infty$.
\end{theorem}

{\bf Interpretation of Why DSC Beats Pure Likelihood Penalties under Overfit.}
Heights detect structural redundancy: small $h_N^{(\kappa)}$ indicates that the $\kappa$-atom model contains a pair of
nearly indistinguishable atoms (often created by over-specification), which standard criteria (AIC/BIC/ICL) do not
directly penalize. DSC combines this structural signal with fit, yielding a sweep-free selector aligned with the Voronoi
geometry.

\section{Numerical Experiment}\label{sec_numerical_experiment}

In this section, we conduct experiments on both synthetic and real-world datasets to systematically evaluate the effectiveness, stability, and practical utility of the proposed pipeline: batch MM training, Voronoi-loss diagnostics, and dendrogram-based selection/merging via DSC. 
On \textit{synthetic datasets}, we first demonstrate that the proposed Batch MM algorithm converges reliably to the global optimum, while classical optimization methods, including EM and popular gradient-based optimizers, often become trapped at suboptimal stationary points, in both well-specified and over-parameterized regimes. 
We then empirically verify that the decay of the Voronoi loss aligns with the finite-sample convergence rates established in \cref{thm_ConvergenceRate_height_Voronoi}, under both exact-fit and over-fit settings. 
Finally, we show that the proposed DSC-based merging strategy consistently outperforms standard model selection criteria (AIC, BIC, and ICL) in recovering the true number of experts, especially in small-sample regimes, while retaining consistency as sample size grows.
On \textit{real-world datasets}, we show that SGMLMoE trained via Batch MM yields superior predictive performance compared to strong baselines such as SVM and Random Forest. Moreover, DSC selects an appropriate number of experts that achieves an attractive trade-off between predictive accuracy and memory efficiency, while the resulting dendrogram provides an interpretable hierarchical summary of the fitted mixing measure via \cref{fig:Tree}.

\begin{figure}
    \centering
    \includegraphics[width=\linewidth]{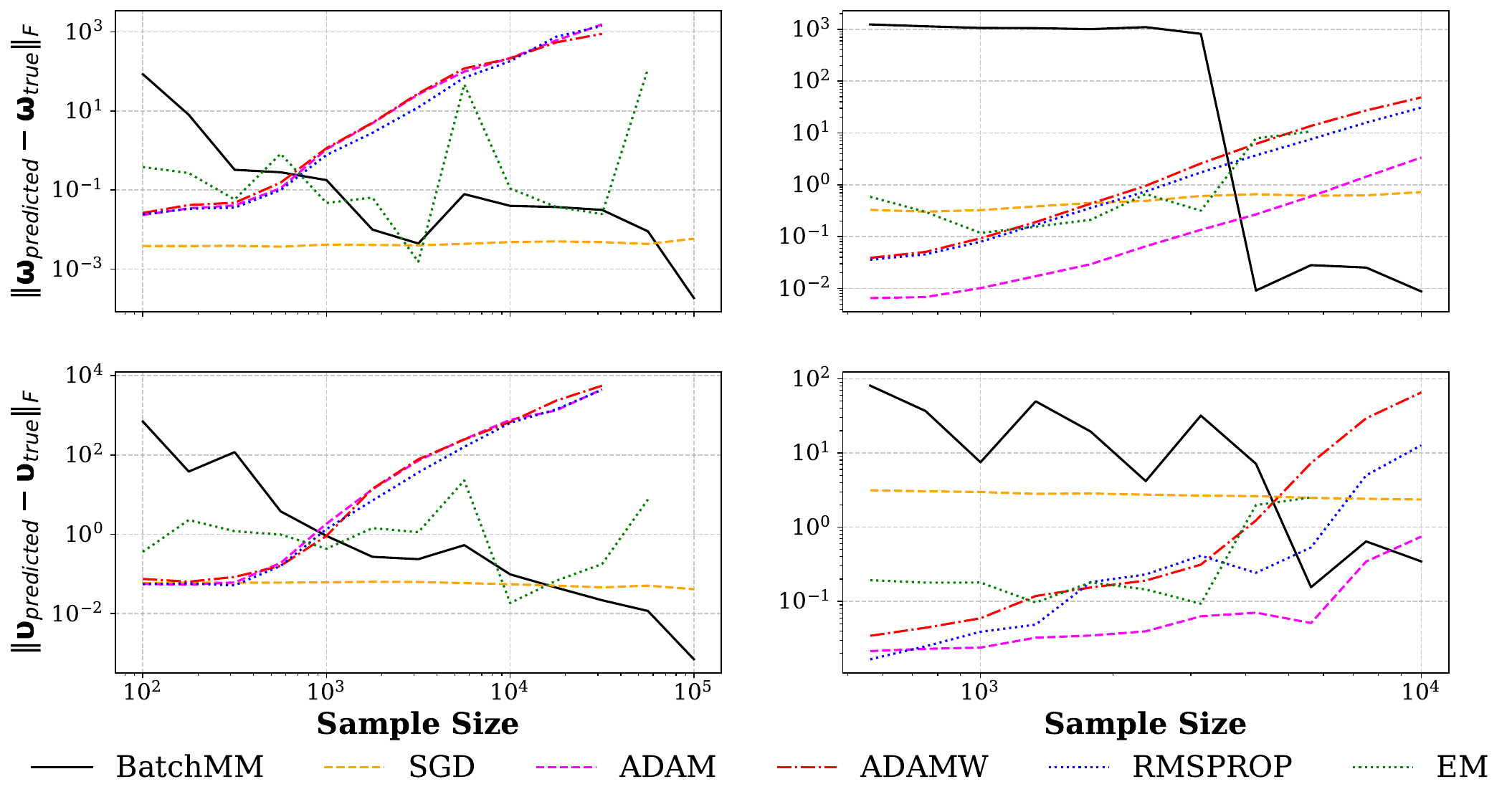}
    \caption{Convergence behavior of the Batch MM algorithm (parameter error versus sample size).}
    \label{fig:BatchMM_Converge}
\end{figure}
\begin{figure*}[t]
    \centering
    \begin{subfigure}[t]{0.48\linewidth}
        \centering
        \includegraphics[width=.82\linewidth]{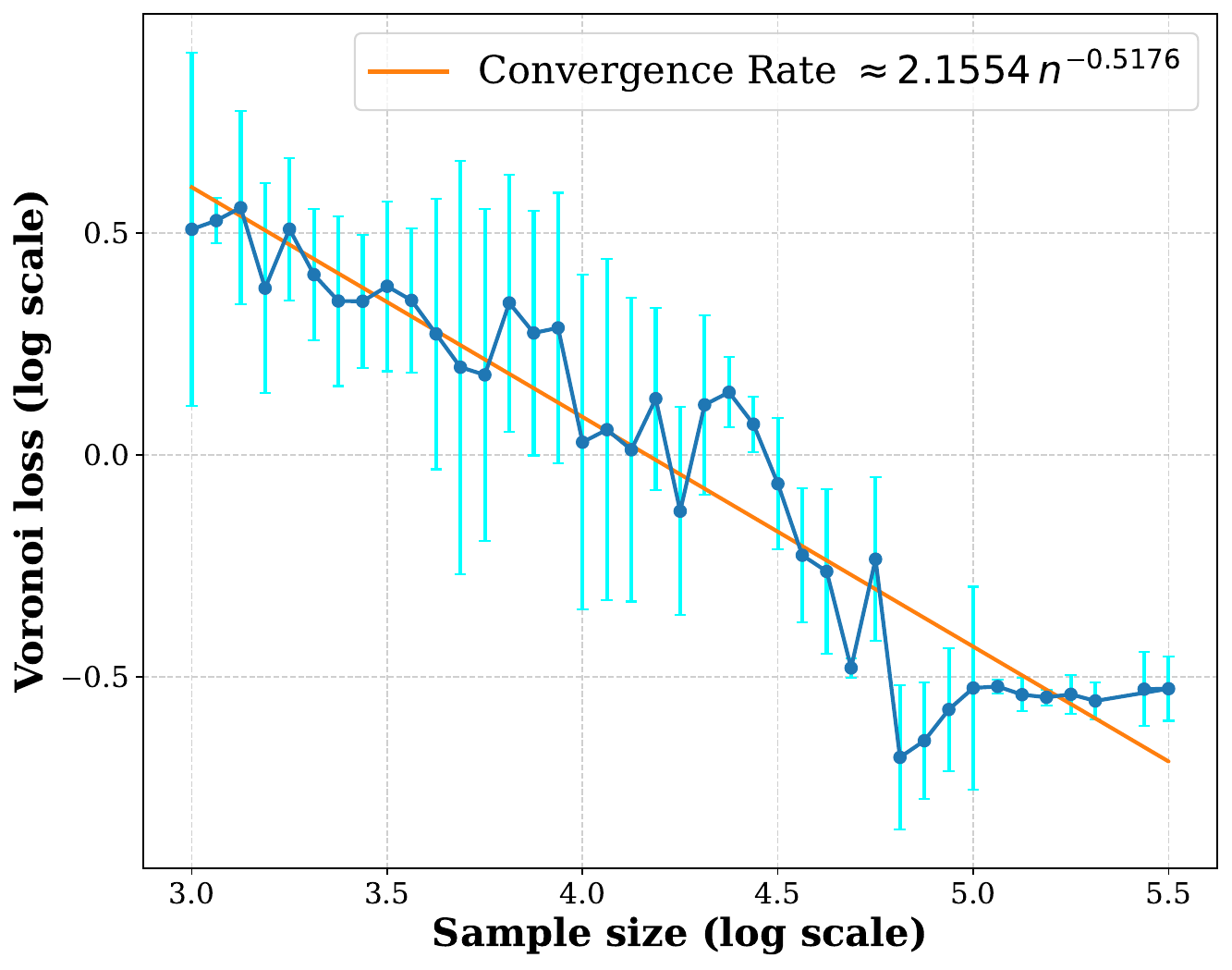}
        \caption{Exact-fit: empirical convergence of the Voronoi loss.}
        \label{fig:VoronoiCov_fit}
    \end{subfigure}
    \begin{subfigure}[t]{0.48\linewidth}
        \centering
        \includegraphics[width=\linewidth]{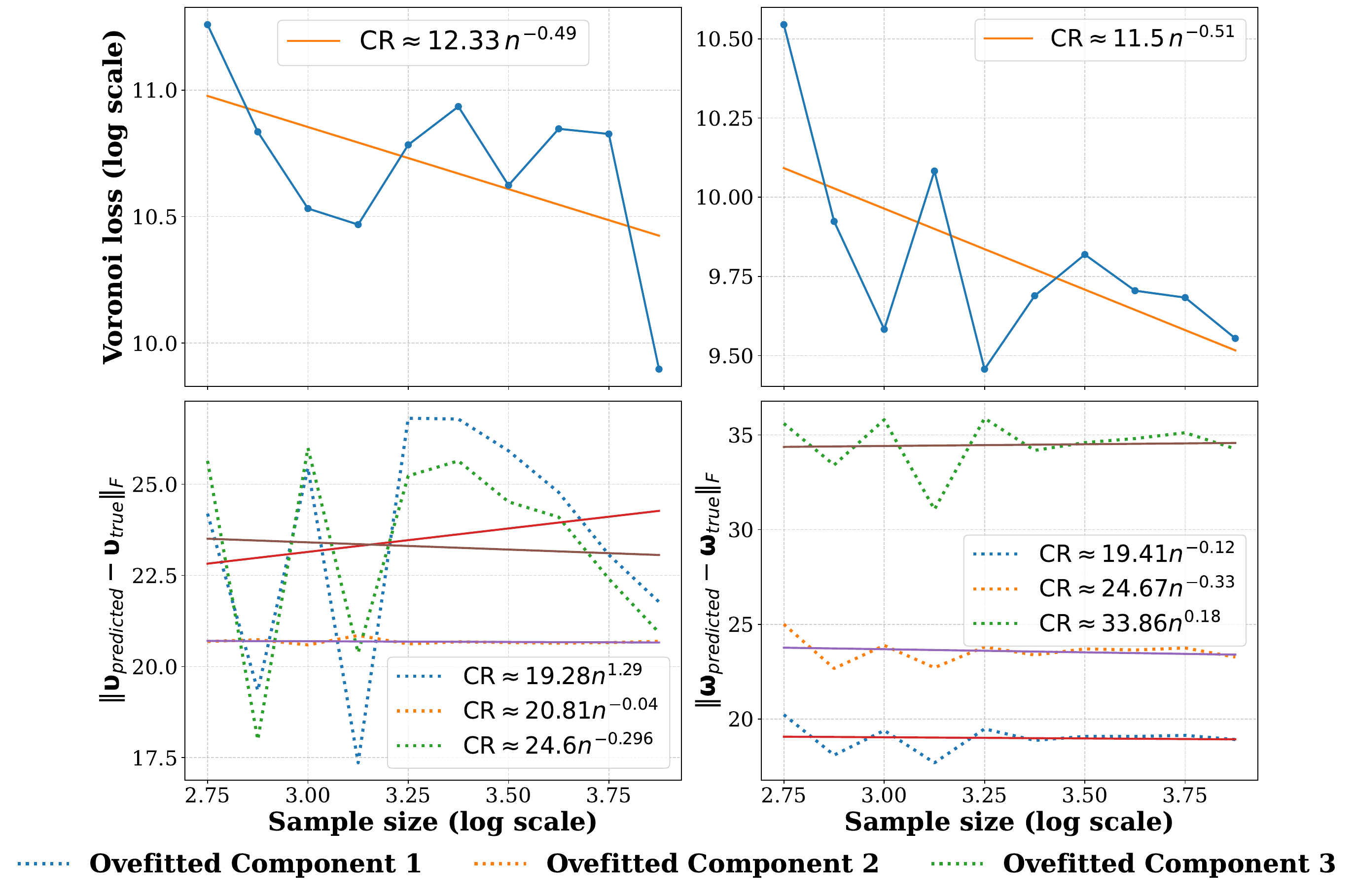}
        \caption{Over-fit: empirical convergence of the Voronoi loss.}
        \label{fig:VoronoiCov_overfit}
    \end{subfigure}
    \caption{Empirical convergence of the Voronoi loss under exact specification and over-specification.}
    \label{fig:VoronoiCov_both}
\end{figure*}

\subsection{Simulation Study}
\textbf{Dataset.}
We consider a controlled synthetic setting with polynomial degree $D = 1$, number of experts $K_0 = 2$, number of categories $M = 2$, and input dimension $P = 1$.
The ground-truth gate parameters are $\vomega_{0,0,0}=\vomega_{1,0,0}=\vomega_{1,1,0}=0$ and $\vomega_{0,0,1}=8$, and the expert parameters are $\vupsilon_{0,0,0,0}=\vupsilon_{1,0,1,0}=\vupsilon_{0,1,0,0}=\vupsilon_{1,1,1,0}=0$, $\vupsilon_{0,0,1,0}=\vupsilon_{0,1,1,0}=20$, and $\vupsilon_{1,0,0,0}=-\vupsilon_{1,1,0,0}=10$.
To generate the dataset $(\sX,\sY)$, we first fix a sample size $N$ (which varies across experiments). We then sample inputs
$\sX=\{x_n\}_{n=1}^N$ i.i.d.\ from the standard normal distribution $\mathcal{N}(0,1)$. 
Conditional on $\sX$, we generate discrete outputs $\sY=\{y_n\}_{n=1}^N$ according to the SGMLMoE model specified by $(\vomega_0,\vupsilon_0)$.

\textbf{Initialization.}
In the well-specified setting ($K=K_0$), parameters are initialized by perturbing the ground-truth values $\vomega_0$ and $\vupsilon_0$ with additive Gaussian noise: each coordinate is set to the corresponding ground-truth value plus an independent $\mathcal{N}(0,1)$ perturbation, scaled by a prescribed noise level.
In the over-parameterized setting ($K>K_0$), we introduce additional expert components beyond the truth.
For each extra cluster $k\in\{K_0+1,\ldots,K\}$, we set $\vomega_{k,0,0}=0$, $\vomega_{k,0,1}=8$, $\vupsilon_{0,k,0,0}=\vupsilon_{1,k,1,0}=0$, $\vupsilon_{0,k,1,0}=20$, and $\vupsilon_{1,k,0,0}=10$.
The original $K_0$ experts are initialized as in the well-specified case. This design ensures that over-parameterization genuinely creates near-duplicate atoms, which is precisely the regime targeted by our Voronoi analysis and merging strategy.

\textbf{Convergence of BatchMM.}
We study the convergence behavior of SGMLMoE trained using the proposed Batch MM algorithm, and compare it against classical optimization methods, including stochastic gradient descent (SGD), Adam, AdamW, RMSprop, and EM.
In the well-specified setting, we vary the sample size from $N=10^2$ to $N=10^5$.
In the over-parameterized setting, we fix $K=3$ and vary $N$ from $10^{2.75}$ to $10^4$.
Convergence is evaluated via the Frobenius norm between the estimated parameters $(\vomega,\vupsilon)$ and the ground-truth $(\vomega_0,\vupsilon_0)$.
As shown in \cref{fig:BatchMM_Converge}, Batch MM consistently converges to the global minimum, while the competing methods are often markedly slower and, in the over-parameterized regime, frequently converge to suboptimal solutions as $N$ grows.

\textbf{Convergence of Voronoi Loss.}
We empirically verify the convergence rates predicted by \cref{thm_ConvergenceRate_height_Voronoi}. In the exact-fit setting (\cref{fig:VoronoiCov_fit}), varying the sample size from $10^3$ to $10^{5.5}$ yields a log–log slope consistent with the parametric rate $N^{-1/2}$. In the over-specified setting with $K\in\{3,4\}$ and $N\in[10^{2.5},10^4]$, the fitted regression lines in \cref{fig:VoronoiCov_overfit} match the predicted rates. For $K=4$, the Frobenius error of individual components shows that one spurious component converges at a slower sub-parametric rate, empirically close to $N^{-0.25}$, in agreement with the rate predicted by \cref{thm_ConvergenceRate_height_Voronoi} for over-covered Voronoi cells. In contrast, another overfitted component exhibits increasing error as $N$ grows, reflecting a mild repulsion effect in over-parameterized models whereby redundant components are displaced away from the true parameters to resolve non-identifiability (here the relevant cells are $\sV_k$).

\textbf{Merging using DSC.}
Finally, we evaluate the effectiveness of our DSC-based merging strategy for recovering the correct number of experts.
We fix $K=4$ and vary the sample size from $N=10^2$ to $10^4$. For each $N$, we repeat the experiment 20 times with different random seeds.
We compare DSC against AIC, BIC, and ICL by recording (i) the frequency with which each criterion selects the correct number of experts and (ii) the average number of selected experts; results are shown in \cref{fig:VoronoiCov_overfit_DSC}.
Across sample sizes, DSC exhibits a higher correct-selection frequency than the baseline criteria. In particular, in the small-sample regime, DSC is notably more robust, while for sufficiently large $N$ it consistently recovers the true number of experts, matching the consistency guarantees established in our theory.

\begin{figure}
    \centering
    \includegraphics[width=\linewidth]{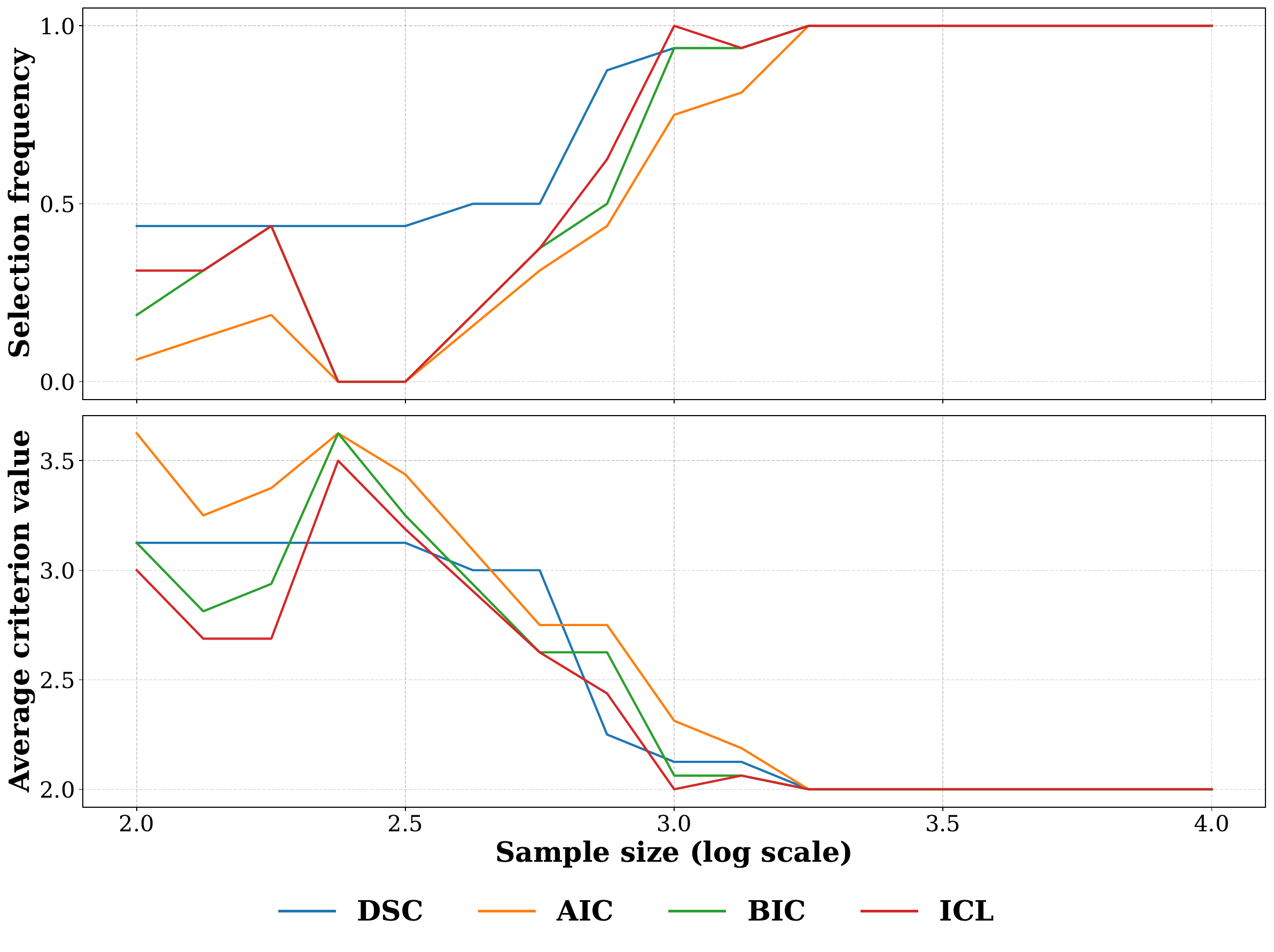}
    \caption{DSC model selection performance.}
    \label{fig:VoronoiCov_overfit_DSC}
\end{figure}

\subsection{Real Dataset}\label{sec: RealData}


\textbf{Initialization.}
Since the true latent expert assignments are unknown in real data, we adopt a practical two-stage initialization strategy grounded in unsupervised learning.
We first apply a classical clustering method (e.g., K-means or spectral clustering) to obtain preliminary groupings.
We then construct initial SGMLMoE parameters from these assignments, providing a structured, data-informed initialization for Batch MM.

\textbf{Expert Estimation.}
We use a subset of 5,000 samples from a protein–protein interaction prediction dataset \citep{tang_machine_2023,qi_mixture_2007} to train the proposed DSC method as well as baseline criteria including AIC, BIC, and ICL for selecting the number of experts, with the maximum capped at $K=8$.
A separate set of 5,000 samples is then used for model training and cross-validation across candidate values of $K$. \cref{tab:perf_by_k_transposed} reports cross-validation performance as a function of $K$, along with the expert count selected by each criterion.
DSC selects $K=3$, achieving strong predictive performance while maintaining a compact model. In contrast, ICL favors a smaller model ($K=2$), leading to reduced precision and F1 score, whereas AIC and BIC select larger models ($K=4$), increasing complexity with only marginal performance gains.
Overall, these results show that DSC achieves a better trade-off between accuracy and model complexity than classical criteria.

\textbf{Accuracy Comparision.}
We fix the number of experts to $K=3$ and train SGMLMoE using the Batch MM algorithm on 5,000 samples.
We compare its performance with standard baseline methods, Naive Bayes, Random Forest, Support Vector Machine, and Logistic Regression, using recall, precision, and F1 score.
As shown in \cref{tab:baseline_compare}, SGMLMoE achieves the best overall performance among all methods, attaining the highest precision ($0.74$) and F1 score ($0.80$), while maintaining competitive recall ($0.88$).
Naive Bayes exhibits lower precision and F1 score, indicating limited discriminative capacity, while Random Forest, SVM, Logistic Regression, and 3-Hidden Layer MLP improve upon NB but remain consistently below SGMLMoE in precision and overall F1.
These results underscore the benefit of SGMLMoE in capturing heterogeneity, and they empirically support the practical stability of Batch MM relative to standard training baselines.
Finally, \cref{fig:Tree} displays the dendrogram of the fitted mixing measure obtained by \cref{alg_batchMM_SGMLMoE}, which provides an interpretable hierarchical view of the inferred expert structure.

\begin{figure}
    \centering
    \includegraphics[width=.7\linewidth]{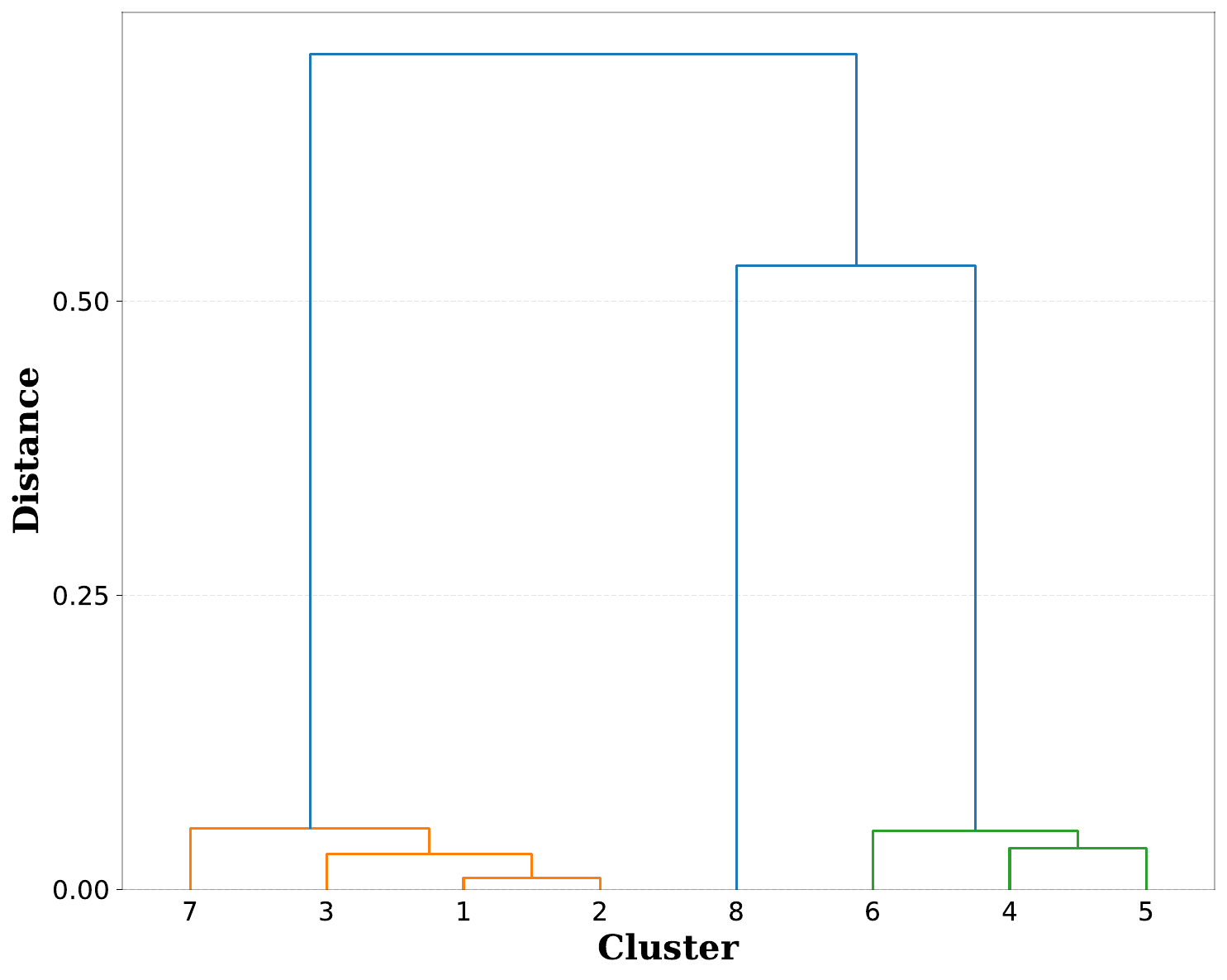}
    \caption{Dendrogram of mixing measure.}
    \label{fig:Tree}
\end{figure}

\begin{table}
    \centering
\begin{tabular}{c c c c c | c}
\toprule
$K$ & Recall & Precision & F1 & \# P & CoMs \\
\midrule
1 & 0.54\,{\scriptsize$\pm$0.04}  & 0.43\,{\scriptsize$\pm$0.03} & 0.48\,{\scriptsize$\pm$0.003} & 6 & \\
2 & 0.72\,{\scriptsize$\pm$0.02}  & 0.61\,{\scriptsize$\pm$0.02} & 0.66\,{\scriptsize$\pm$0.003} & 12 & ICL \\
3 & 0.88\,{\scriptsize$\pm$0.01}  & 0.73\,{\scriptsize$\pm$0.02} & 0.80\,{\scriptsize$\pm$0.003} & 18 & DSC\\
4 & 0.88\,{\scriptsize$\pm$0.01}  & 0.74\,{\scriptsize$\pm$0.01} & 0.80\,{\scriptsize$\pm$0.002} & 24 & AIC, BIC\\
5 & 0.89\,{\scriptsize$\pm$0.01}  & 0.75\,{\scriptsize$\pm$0.01} & 0.81\,{\scriptsize$\pm$0.002} & 30 &\\
6 & 0.91\,{\scriptsize$\pm$0.01}  & 0.76\,{\scriptsize$\pm$0.01} & 0.83\,{\scriptsize$\pm$0.001} & 36 &\\
7 & 0.92\,{\scriptsize$\pm$0.01}  & 0.76\,{\scriptsize$\pm$0.01} & 0.83\,{\scriptsize$\pm$0.001} & 42 &\\
8 & 0.92\,{\scriptsize$\pm$0.01}  & 0.76\,{\scriptsize$\pm$0.01} & 0.83\,{\scriptsize$\pm$0.001} & 48 &\\
\bottomrule
\end{tabular}
    \caption{Performance across number of experts $K$.}
    \label{tab:perf_by_k_transposed}
\end{table}

\begin{table}
    \centering
\caption{Comparison with baseline methods.}
\label{tab:baseline_compare}
\begin{tabular}{l c c c}
\toprule
Method & Recall & Precision & F1 \\
\midrule
SGMLMoE
& \textbf{0.88}\,{\scriptsize$\pm$0.012}
& \textbf{0.74}\,{\scriptsize$\pm$0.010}
& \textbf{0.80}\,{\scriptsize$\pm$0.002} \\

NB
& 0.83\,{\scriptsize$\pm$0.011}
& 0.64\,{\scriptsize$\pm$0.009}
& 0.72\,{\scriptsize$\pm$0.001} \\

RF
& 0.85\,{\scriptsize$\pm$0.012}
& 0.68\,{\scriptsize$\pm$0.008}
& 0.76\,{\scriptsize$\pm$0.001} \\

SVM
& 0.85\,{\scriptsize$\pm$0.010}
& 0.69\,{\scriptsize$\pm$0.009}
& 0.76\,{\scriptsize$\pm$0.002} \\

LR
& 0.84\,{\scriptsize$\pm$0.014}
& 0.65\,{\scriptsize$\pm$0.009}
& 0.73\,{\scriptsize$\pm$0.003} \\

3HL-MLP
& 0.87\,{\scriptsize$\pm$0.008}
& 0.71\,{\scriptsize$\pm$0.009}
& 0.77\,{\scriptsize$\pm$0.001} \\
\bottomrule
\end{tabular}
\end{table}

{\bf Conclusion and Perspectives.}
In summary, our experiments confirm the paper’s main message: Batch MM yields stable maximum-likelihood training for SGMLMoE in the full-data regime, while the Voronoi-loss/dendrogram view provides a principled remedy for over-parameterization. Empirically, Batch MM converges reliably, the Voronoi loss decays in line with the predicted finite-sample rates, and DSC consistently selects a compact expert count that preserves accuracy and improves calibration on protein--protein interaction prediction. Future work includes extending the MM surrogates to incremental/mini-batch settings, broadening the dendrogram theory to richer expert parameterizations and structured regularization, and developing robustness guarantees under contamination and covariate shift.

\newpage

\section*{Impact Statement}
This paper presents work whose goal is to advance the field of Machine
Learning. There are many potential societal consequences of our work, none which we feel must be specifically highlighted here

\bibliography{bib}
\bibliographystyle{icml2026}

\newpage
\appendix
\onecolumn

\icmltitle{Supplementary Materials for ``Fast Model Selection and Stable Optimization for Softmax-Gated Multinomial-Logistic Mixture of Experts Models"}

\paragraph{Appendix Organization.}
This appendix collects material that supports the main paper at three complementary levels. \cref{appendix_extended_conclusion} extends the discussion with empirical takeaways and concrete future directions. \cref{app:batchMM_SGMLMoE_details} provides implementation-ready details of the Batch MM routine for SGMLMoE, including the linear-algebra identities, the full algorithmic steps, and practical notes on conditioning and stopping. \cref{appendix_exp_details} documents experimental protocols (data preprocessing, evaluation, and reproducibility details). Finally, \cref{appendix_proof_main_results} contains complete proofs.

\section{Extended Conclusion and Perspectives}\label{appendix_extended_conclusion}
Overall, the numerical experiments support the main message of the paper: in SGMLMoE, the proposed Batch MM updates provide a stable and reliable maximum-likelihood training routine in the full-data regime, and the Voronoi-loss/dendrogram viewpoint yields a practically effective handle on over-parameterization. Empirically, we observe (i) monotone and repeatable likelihood improvement with Batch MM across well-specified and over-specified settings, (ii) decay of the Voronoi loss consistent with the finite-sample theory in both exact-fit and over-fit regimes, and (iii) robust recovery of the effective number of experts via DSC, which translates into compact models with competitive or improved predictive performance and probability calibration on protein--protein interaction data.

Looking forward, several extensions are natural. First, it would be valuable to couple the present batch-MM stability guarantees with incremental or mini-batch variants to handle larger-scale datasets while preserving the monotone-surrogate structure. Second, extending the dendrogram/DSC analysis to richer expert parameterizations (e.g., deeper feature maps, structured sparsity, or context-dependent expert sharing) could broaden applicability while maintaining interpretability through the merging hierarchy. Third, the robustness observed under mild misspecification suggests studying contamination and covariate-shift regimes theoretically, with DSC-style structural penalties as a principled alternative to purely likelihood-based criteria. Finally, developing open-source, production-grade implementations and broader biological benchmarks (beyond PPI) would further clarify when and how hierarchical aggregation best trades off accuracy, calibration, and memory/compute in heterogeneous classification problems.

\section{Detailed Batch MM Algorithm}\label{app:batchMM_SGMLMoE_details}

{\bf Inverse identities.}
For implementation, we use the block structure of \cref{eq_B_kn_batch}. In particular,
\begin{align}\label{eq_inverse_B_batch}
    \mB_{n,K-1}^{-1}
    &=\left(\frac{4}{3}\mI_{K-1}+\frac{8}{3(K-1)}\bm{1}_{K-1}\bm{1}_{K-1}^{\top}\right)\otimes\left(\sum_{n=1}^{N}\hat{\vx}_n\hat{\vx}_n^{\top}\right)^{-1},\\
    \left[\mB^{K}_{n,M-1}\right]^{-1}&=\mI_{K}\otimes \mB_{n,M-1}^{-1}.
\end{align}
In addition, for any diagonal matrix $\diag(\va)$ with positive entries, we have $\diag(\va)^{-1}=\diag(1/\va)$, which we apply to $\mZ_n^{(t)}$ when needed.

{\bf Batch MM Algorithm for SGMLMoE.} \cref{alg_batchMM_SGMLMoE} implements a batch MM procedure for fitting SGMLMoE. At each iteration, it first computes responsibilities ${\tau_{n,k}^{(t)}}$ for all samples and experts using the current gate and expert probabilities, and then constructs the batch sufficient-statistic–like terms $(\vs_n^{(t)},\vr_n^{(t)})$ together with the corresponding curvature blocks $(\mB_{n,K},\mB_{n,M})$ that define the quadratic MM surrogate in \cref{eq_surrogate_batch_SGMLMoE}. These per-sample quantities are then aggregated over the batch to form global updates, including the diagonal weighting matrices $\mZ_n^{(t)}$ and the gradients of the log-sum-exp terms $g_n$ (gate) and $e_n$ (experts). Finally, the method performs closed-form MM updates for the gate and expert parameters by solving two preconditioned linear systems, yielding $(\vw^{(t+1)},\vv^{(t+1)})$. By construction, the resulting iterates monotonically improve the log-likelihood, and the routine terminates when the likelihood increment falls below the tolerance $\varepsilon$ (or when the maximum number of iterations is reached).

\begin{algorithm}
\caption{{\bf Batch MM Algorithm for SGMLMoE}}
\label{alg_batchMM_SGMLMoE}
\begin{algorithmic}[1]
\Require Batch data $\{(\vx_n,\ry_n)\}_{n=1}^{N}$, numbers of experts/classes $(K,M)$, polynomial degree $D$, initialization $\vtheta^{(0)}=(\vw^{(0)},\vv^{(0)})$, tolerance $\varepsilon>0$, maximum iterations $T$.
\Ensure MM iterates $\{\vtheta^{(t)}\}_{t\ge 0}$.

\For{$t=0,1,\dots,T-1$}
    \State \textbf{E-step (Responsibilities).}
    \For{$n=1,\dots,N$}
        \State Form the feature vector $\hat{\vx}_n=\big[\evx_{n,1}^{0},\dots,\evx_{n,1}^{D},\dots,\evx_{n,P}^{0},\dots,\evx_{n,P}^{D}\big]^{\top}$.
        
        \State Compute gate probabilities $\{\sfg_k(\vw^{(t)}(\vx_n))\}_{k=1}^{K}$.
        
        \State Compute expert probabilities $\{\ex_k(\ry_n=m;\vv^{(t)}(\vx_n))\}_{m=1}^{M}$ for each $k\in[K]$.
        
        \State Compute posteriors
        \[
        \tau_{n,k}^{(t)}=
        \frac{\sfg_k(\vw^{(t)}(\vx_n))\,\ex_k(\ry_n;\vv^{(t)}(\vx_n))}
        {\sum_{\ell=1}^{K}\sfg_\ell(\vw^{(t)}(\vx_n))\,\ex_\ell(\ry_n;\vv^{(t)}(\vx_n))},~ k\in[K].
        \]
        
        \State Build $\vs_n^{(t)}$ and $\vr_n^{(t)}$ (using $\I(\ry_n,l)$ and the definitions in \cref{eq_surrogate_batch_SGMLMoE} and the surrounding text).
        
        \State Build the curvature blocks $\mB_{n,K}$ and $\mB_{n,M}$ (as in \cref{eq_B_kn_batch}).
    \EndFor

    \State \textbf{Aggregate Batch Quantities.}
    
    \State $\vs^{(t)}=\sum_{n=1}^{N}\vs_n^{(t)}$, \quad $\vr^{(t)}=\sum_{n=1}^{N}\vr_n^{(t)}$.
    
    \State $\mB_{n,K-1}=\sum_{n=1}^{N}\mB_{n,K}$,\quad $\mB_{n,M-1}=\sum_{n=1}^{N}\mB_{n,M}$,\quad $\mB^{K}_{n,M-1}=\mI_{K}\otimes \mB_{n,M-1}$.
    
    \For{$n=1,\dots,N$}
        \State Set $\vtau_{n,:}^{(t)}=[\tau_{n,1}^{(t)},\dots,\tau_{n,K}^{(t)}]^{\top}$ and
        $\mZ_n^{(t)}=\diag\!\big(\vtau_{n,:}^{(t)}\otimes \bm{1}_{P(D+1)}\big)$.
    \EndFor

    \State \textbf{Compute Gradients of Log-sum-exp terms.}
    \State $\nabla g(\vw^{(t)})=\sum_{n=1}^{N}\nabla g_n(\vw^{(t)})$, where $g_n(\vw)=\log\!\left(1+\sum_{k=1}^{K-1}\exp\!\left(w_k(\vx_n)\right)\right)$.
    \State For each $n\in[N]$, form $\nabla e_n(\vv^{(t)})=\vect\big([\nabla e_n(\vc_1^{(t)}),\dots,\nabla e_n(\vc_K^{(t)})]\big)$, where $e_n(\vc_k)=\log\!\left(1+\sum_{m=1}^{M-1}\exp\!\left(\evv_{m,k}(\vx_n)\right)\right)$.

    \State \textbf{M-step (Closed-form MM Updates).}
    \State Update the gate parameters:
    \[
        \vw^{(t+1)}=\vw^{(t)}+\mB_{n,K-1}^{-1}\Big(\vs^{(t)}-\nabla g(\vw^{(t)})\Big).
    \]
    \State Update the expert parameters:
\begin{align*}
    &\vv^{(t+1)}=\vv^{(t)}+\Big(\sum_{n=1}^{N}\mZ_n^{(t)}\,\mB^{K}_{n,M-1}\Big)^{-1}\Big(\vr^{(t)}-\sum_{n=1}^{N}\mZ_n^{(t)}\,\nabla e_n(\vv^{(t)})\Big).
\end{align*}
    \State Set $\vtheta^{(t+1)}=(\vw^{(t+1)},\vv^{(t+1)})$.

    \State \textbf{Stopping Criterion.} If $|\mathcal{L}(\vtheta^{(t+1)})-\mathcal{L}(\vtheta^{(t)})|\le \varepsilon$, \textbf{break}.
\EndFor
\end{algorithmic}
\end{algorithm}

\section{Experimental Details}\label{appendix_exp_details}

\textbf{Real Dataset Pre-processing. } We first use the dataset provided in \citep{tang_machine_2023,qi_mixture_2007} to obtain positive interaction labels. Following \citet{tang2023machine}, we adopt protein sequence data as inputs, collected from \citet{uniprot2015uniprot}, yielding 9,760 protein sequences. Each sequence is encoded as a 20-dimensional amino-acid composition vector using the ProtLearn library. We then restrict the dataset to proteins appearing in the DIP database, resulting in 2,321 proteins with retained identifiers and corresponding sequence encodings. To construct labeled interaction samples, we select $n$ protein pairs with verified interactions in DIP and compute the element-wise absolute difference between their sequence encodings, producing 20-dimensional feature vectors labeled as 1. Similarly, we sample $n$ non-interacting protein pairs absent from DIP and construct their feature vectors in the same manner, assigning label 0. The positive and negative samples are concatenated to form a $2n \times 20$ feature matrix with an associated binary label vector, and the resulting dataset is randomly shuffled prior to training.

\textbf{Monotonicity of Maximum Loglikelihood Estimator. }Classical optimization methods such as SGD and Adam often lead to unstable MLE behavior in mixture or Mixture-of-Experts models. This instability arises from the highly non-convex likelihood landscape, where inappropriate step sizes can cause oscillations or divergence, making performance sensitive to careful learning-rate tuning. Such tuning is problem-dependent and can be computationally costly, limiting robustness and reproducibility. In contrast, our proposed method avoids these issues by adopting an MM-based optimization scheme that guarantees monotonic improvement of the MLE objective, as established in \cref{thm_MM_monotone_SGMLMoE}. This monotonicity ensures stable and predictable optimization dynamics. We empirically illustrate this behavior by plotting the MLE trajectory across iterations in \cref{fig_MLE_Mono}. As shown, our method exhibits a smooth, non-decreasing likelihood progression, in contrast to the fluctuating behavior observed under gradient-based optimizers. The experiment is conducted on the synthetic dataset described in \cref{sec: RealData} with $N=1000$ samples.

\begin{figure}
    \centering
    \includegraphics[width=.7\linewidth]{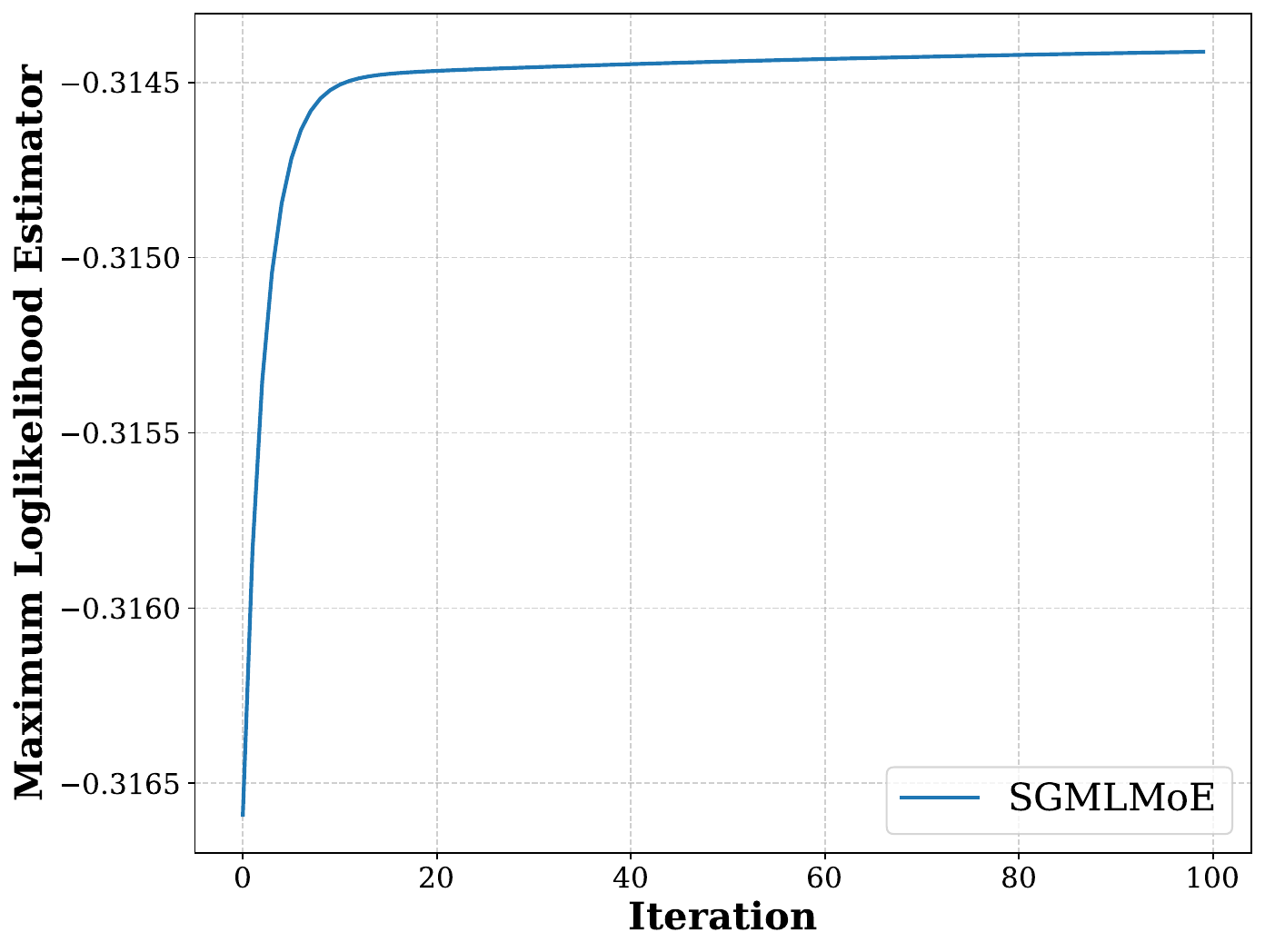}
    \caption{Evolution of the maximum log-likelihood for SGMLMoE over the first 100 optimization iterations with $N=1000$.}
    \label{fig_MLE_Mono}
\end{figure}
Surrogate function construction plays a pivotal role in optimization techniques such as MM algorithms. By leveraging mathematical inequalities, it is possible to approximate or bound complex objective functions with simpler, more tractable alternatives. Below are several fundamental inequalities commonly employed in the development of surrogate functions:


\section{Proofs of Main Results}\label{appendix_proof_main_results}

\subsection{Technical Contributions} 

Following the MM paradigm, we first propose a novel surrogate objective for SGMLMoE in \cref{theorem_surrogate} and rigorously show that it satisfies the monotonicity property required by the MM algorithm in \cref{thm_MM_monotone_SGMLMoE}. Building on this foundation, we turn to model merging and substantially modify the proof of Theorem~3.1 in \citet{nguyen_general_2024} to accommodate our proposed Voronoi metric and the SGMLMoE framework, yielding the key inequality in \cref{thm_VorIneq_SGMLMoE}. Leveraging this result, we then follow the argument of Lemma~2 in \citet{do_dendrogram_2024} (see also \citet{thai_model_2025,hai_dendrograms_2026}) to derive an analogous monotone chain property for our model in \cref{thm_VorChain_SGMLMoE}. Finally, we adapt the proofs of Theorems~1--3 in \citet{do_dendrogram_2024} for mixture models to our Voronoi loss and SGMLMoE setting, leading to theoretical guarantees on the convergence rate, model selection consistency, and merge convergence in \cref{thm_ConvergenceRate_height_Voronoi}, \cref{thm_ModelSelection_SGMLMoE}, and \cref{thm:ConvergeMerge_SGMLMoE}, respectively.

\subsection{Proof of \texorpdfstring{\cref{theorem_surrogate}}{Theorem \ref{theorem_surrogate}}}
\label{sec:Surrogate_SGMLMoE}

We derive the surrogate bound in \cref{eq_surrogate_batch_SGMLMoE} by combining
(i) a variational upper bound based on responsibilities (E-step) and
(ii) a uniform quadratic majorization of the baseline log-sum-exp terms (MM step).
We first recall the notation blocks used throughout this proof.

{\bf Batch Sufficient-statistics Notation.}
For each $n\in[N]$, define the gating-design vector in $\sR^{(K-1)P(D+1)}$:
\begin{align}\label{eq_vs_n_batch_recall}
    \vs_n^{(t)}=\big[\tau_{n,1}^{(t)}\evx_{n,1}^{0},\dots,\tau_{n,1}^{(t)}\evx_{n,1}^{D},\dots,\tau_{n,K-1}^{(t)}\evx_{n,P}^{D}\big]^{\top}.
\end{align}
Define the expert-design vector $\vr_n^{(t)}$ by stacking
$r_{n,d,l,p,k}^{(t)}=\I(\ry_n,l)\,\tau_{n,k}^{(t)}\,\evx_{n,p}^{d}$,
$~d\in\{0,\ldots,D\},~l\in[M],~p\in[P],~k\in[K],$
into
\[
\vr_n^{(t)}=
\big[r_{n,0,1,1,1}^{(t)},\dots,r_{n,D,1,1,1}^{(t)},\dots,r_{n,D,M,1,1}^{(t)},\dots,r_{n,D,M,P,K}^{(t)}\big]^{\top}.
\]
We also define the class-parameter blocks for each expert $k$ by
\[
\vc_k=\big[\vupsilon_{1,k,0,1},\dots,\vupsilon_{1,k,D,1},\dots,\vupsilon_{M,k,D,1},\dots,\vupsilon_{M,k,D,P}\big]^{\top},
\qquad
\vupsilon=\vect([\vc_1,\dots,\vc_K]).
\]
Finally, we aggregate the batch statistics as
\[
\vs^{(t)}=\sum_{n=1}^{N}\vs_n^{(t)},
\qquad
\vr^{(t)}=\sum_{n=1}^{N}\vr_n^{(t)}.
\]

{\bf Quadratic Majorizers for the Gate and Experts.}
For $k\in[K-1]$, let $w_k(\vx)=\sum_{d=0}^{D}\vomega_{k,d}^{\top}\vx^{d}$ and define the gate log-sum-exp term
\[
g_n(\vw)=\log\!\left(1+\sum_{k=1}^{K-1}\exp\!\left(w_k(\vx_n)\right)\right).
\]
For experts, define for each $k\in[K]$ the multinomial log-sum-exp term
\[
e_n(\vc_k)=\log\!\left(1+\sum_{m=1}^{M-1}\exp\!\left(\evv_{m,k}(\vx_n)\right)\right).
\]
For any $K\in\sN$, we use the quadratic majorizers based on the matrices
\begin{align}\label{eq_B_kn_batch_recall}
    \mB_{n,K}
    =\left(\frac{3}{4}\mI_{K-1}-\frac{\bm{1}_{K-1}\bm{1}_{K-1}^{\top}}{2(K-1)}\right)\otimes\hat{\vx}_n\hat{\vx}_n^{\top},
\end{align}
and similarly $\mB_{n,M}$ is defined by replacing $K$ with $M$ in \cref{eq_B_kn_batch_recall}.
For compactness, introduce the centered increments $\bar{\vw}_t=\vw-\vw^{(t)}$ and $\bar{\vc}_{k,t}=\vc_k-\vc_k^{(t)}$.
(Here $\hat{\vx}_n=[\evx_{n,1}^{0},\dots,\evx_{n,1}^{D},\dots,\evx_{n,P}^{0},\dots,\evx_{n,P}^{D}]^\top$.)

\medskip

\paragraph{Step 1: Variational Upper Bound via Responsibilities.}
For each $n\in[N]$, define
\[
a_{n,k}(\vtheta):=\sfg_k(\vw(\vx_n))\,\ex_k(\ry_n;\vv(\vx_n)),
\qquad
s_\vtheta(\ry_n\mid \vx_n)=\sum_{k=1}^{K}a_{n,k}(\vtheta).
\]
By the log-sum inequality (equivalently, Jensen's inequality / KL nonnegativity), for any $\vtau_{n,:}\in\Delta_K$,
\begin{equation}\label{eq:logsum_bound_per_n}
-\log s_\vtheta(\ry_n\mid \vx_n)
\le
-\sum_{k=1}^{K}\tau_{n,k}\log a_{n,k}(\vtheta)
+\sum_{k=1}^{K}\tau_{n,k}\log\tau_{n,k}.
\end{equation}
Choosing $\tau_{n,k}=\tau_{n,k}^{(t)}:=a_{n,k}(\vtheta^{(t)})/\sum_{\ell}a_{n,\ell}(\vtheta^{(t)})$ makes
\cref{eq:logsum_bound_per_n} tight at $\vtheta=\vtheta^{(t)}$.

Expanding $\log a_{n,k}(\vtheta)=\log\sfg_k(\vw(\vx_n))+\log\ex_k(\ry_n;\vv(\vx_n))$ yields two convex
log-sum-exp terms and linear contributions in the scores:
\begin{align}
-\sum_{k=1}^{K}\tau_{n,k}^{(t)}\log\sfg_k(\vw(\vx_n))
&=
g_n(\vw)-\sum_{k=1}^{K-1}\tau_{n,k}^{(t)}\,w_k(\vx_n),
\label{eq:gate_decomp}\\
-\sum_{k=1}^{K}\tau_{n,k}^{(t)}\log\ex_k(\ry_n;\vv(\vx_n))
&=
\sum_{k=1}^{K}\tau_{n,k}^{(t)}\,e_n(\vc_k)
-\sum_{k=1}^{K}\tau_{n,k}^{(t)}\sum_{l=1}^{M}\I(\ry_n,l)\,\evv_{l,k}(\vx_n).
\label{eq:expert_decomp}
\end{align}
The linear terms in \cref{eq:gate_decomp,eq:expert_decomp} match the sufficient-statistics inner products
$\vw^\top \vs_n^{(t)}$ and $\vupsilon^\top \vr_n^{(t)}$ under the stacking conventions above
(so that $w_k(\vx_n)=\vomega_k^\top\hat{\vx}_n$ and $\evv_{l,k}(\vx_n)=\vupsilon_{l,k}^\top\hat{\vx}_n$).

Collecting \cref{eq:logsum_bound_per_n,eq:gate_decomp,eq:expert_decomp} and summing over $n$ gives
\begin{align}
-\log s_\vtheta(\ry\mid\rvx)
\le
C^{(t)}
-\sum_{n=1}^{N}\vw^\top \vs_n^{(t)}
-\sum_{n=1}^{N}\vupsilon^\top \vr_n^{(t)}
+\sum_{n=1}^{N} g_n(\vw)
+\sum_{n=1}^{N}\sum_{k=1}^{K}\tau_{n,k}^{(t)}\,e_n(\vc_k),
\label{eq:bound_after_responsibilities}
\end{align}
where $C^{(t)}:=\sum_{n=1}^{N}\sum_{k=1}^{K}\tau_{n,k}^{(t)}\log\tau_{n,k}^{(t)}$ does not depend on $\vtheta$.

\paragraph{Step 2: Quadratic Majorization of $g_n$ and $e_n$.}
Both $g_n$ and $e_n$ are baseline log-sum-exp functions composed with affine score maps.
A standard calculation gives
\[
\nabla^2 g_n(\vw)=\big(\diag(\hat{\vp}_n)-\hat{\vp}_n\hat{\vp}_n^\top\big)\otimes \hat{\vx}_n\hat{\vx}_n^\top,
\]
where $\hat{\vp}_n$ are the baseline softmax probabilities over $k\in[K-1]$.

Applying \cref{prop:bohning_bound_baseline} with $q=K-1$ (gate) and $q=M-1$ (expert), and then lifting the resulting Hessian bound through the linear score maps via \cref{lem:prop:bohning_bound_baseline}, yields the quadratic majorizers in \cref{eq_B_kn_batch}. Indeed, by \cref{lem:prop:bohning_bound_baseline} (with $q=K-1$) and \cref{prop:bohning_bound_baseline},
\[
\nabla^2 g_n(\vw)\ \preceq\ \mB_{n,K},
\]
with $\mB_{n,K}$ defined in \cref{eq_B_kn_batch_recall}. Therefore, the second-order Taylor formula along the segment
between $\vw^{(t)}$ and $\vw$ yields the global quadratic upper bound
\begin{equation}\label{eq:quad_maj_gate}
g_n(\vw)
\le
g_n(\vw^{(t)})+\bar{\vw}_t^\top\nabla g_n(\vw^{(t)})+\frac12\,\bar{\vw}_t^\top \mB_{n,K}\bar{\vw}_t,
\qquad \forall \vw.
\end{equation}
The same argument applies to $e_n(\vc_k)$ (with $q=M-1$), giving for each $k\in[K]$,
\begin{equation}\label{eq:quad_maj_expert}
e_n(\vc_k)
\le
e_n(\vc_k^{(t)})+\bar{\vc}_{k,t}^\top\nabla e_n(\vc_k^{(t)})+\frac12\,\bar{\vc}_{k,t}^\top \mB_{n,M}\bar{\vc}_{k,t},
\qquad \forall \vc_k.
\end{equation}

\paragraph{Step 3: Surrogate Bound.}
Plug \cref{eq:quad_maj_gate,eq:quad_maj_expert} into \cref{eq:bound_after_responsibilities} and absorb all
terms that do not depend on $\vtheta$ into $C^{(t)}$. This yields exactly the surrogate inequality stated in
\cref{eq_surrogate_batch_SGMLMoE}, namely
\[
-\log s_{\vtheta}(\ry\mid \rvx)\le -\log\!\big[g(\vtheta,\vx_{[N]},\ry_{[N]};\vtheta^{(t)})\big],
\]
with the explicit quadratic form in \cref{eq_surrogate_batch_SGMLMoE}. This completes the proof.


\begin{lemma}[Kronecker product preserves Loewner order]\label{lem:prop:bohning_bound_baseline}
If $\mA \preceq \mB$ and $\mC \succeq \zero$ are symmetric, then
\[
\mA \otimes \mC \preceq \mB \otimes \mC .
\]
\end{lemma}
\begin{proof}[Proof of \cref{lem:prop:bohning_bound_baseline}]
Since $\mA \preceq \mB$, we have $\mB-\mA \succeq \zero$. For any vector $\vx$,
\[
\vx^\top\big((\mB-\mA)\otimes \mC\big)\vx
=
\big\| \big((\mB-\mA)^{1/2}\otimes \mC^{1/2}\big)\vx \big\|^2 \ge 0,
\]
so $(\mB-\mA)\otimes \mC \succeq \zero$, which is equivalent to
$\mA \otimes \mC \preceq \mB \otimes \mC$.
\end{proof}

\begin{proposition}[Uniform bound for baseline-softmax Hessian]\label{prop:bohning_bound_baseline}
Let $q\in\sN$ and let $\tilde{\vp}\in(0,1)^{q+1}$ satisfy $\sum_{j=1}^{q+1}\tilde p_j=1$.
Write $\hat{\vp}=[\tilde p_1,\ldots,\tilde p_q]^\top$ and $\mLambda=\diag(\hat{\vp})$.
Then
\[
\mLambda-\hat{\vp}\hat{\vp}^\top
\preceq
\mA_q
:=
\frac{3}{4}\mI_q-\frac{\bm{1}_q\bm{1}_q^\top}{2q}.
\]
\end{proposition}
\begin{proof}[Proof of \cref{prop:bohning_bound_baseline}] In this proof, we use the following definitions:
$\bm{1}_{q+1} := (1,\dots,1)^\top \in \sR^{q+1}$, and
$\bm{1}_{q+1}^\perp := \left\{\vu\in\sR^{q+1}:\ \bm{1}_{q+1}^\top \vu = 0\right\}.$

{\bf Step 1: Full Simplex Covariance Bound.}
Define the $(q\!+\!1)\times(q\!+\!1)$ covariance-type matrix
\[
\mW(\tilde{\vp})=\diag(\tilde{\vp})-\tilde{\vp}\tilde{\vp}^\top.
\]
Let $\rvz\in\{\eb_1,\ldots,\eb_{q+1}\}$ be a one-hot random vector with
$\PE[\rvz]=\tilde{\vp}$. Then $\mW(\tilde{\vp})=\Cov(\rvz)$ and for any $\vu$,
\[
\vu^\top \mW(\tilde{\vp})\vu=\Var(\vu^\top \rvz).
\]
Since $\vu^\top \rvz$ takes values in $\{\vu_1,\ldots,\vu_{q+1}\}$, its range is at most
\[
\max_{i,j}|\vu_i-\vu_j|
=
\max_{i,j}|\vu^\top(\eb_i-\eb_j)|
\le \|\vu\|\max_{i,j}\|\eb_i-\eb_j\|
=\sqrt{2}\,\|\vu\|.
\]
By Popoviciu's inequality, $\Var(\vu^\top\rvz)\le \frac{1}{4}(\text{range})^2 \le \frac12\|\vu\|^2$.
Hence $\mW(\tilde{\vp}) \preceq \frac12\mI_{q+1}$.
Moreover, $\mW(\tilde{\vp})\bm{1}_{q+1}=\zero$, so $\mW(\tilde{\vp})$ is supported on
$\bm{1}_{q+1}^\perp$, and therefore
\[
\mW(\tilde{\vp})
\preceq
\frac12\Big(\mI_{q+1}-\frac{\bm{1}_{q+1}\bm{1}_{q+1}^\top}{q+1}\Big).
\]

{\bf Step 2: Take the Principal Submatrix.}
The top-left $q\times q$ principal submatrix of $\mW(\tilde{\vp})$ is exactly
$\mLambda-\hat{\vp}\hat{\vp}^\top$. Taking the corresponding principal submatrix of the
right-hand side yields
\[
\mLambda-\hat{\vp}\hat{\vp}^\top
\preceq
\frac12\Big(\mI_q-\frac{\bm{1}_q\bm{1}_q^\top}{q+1}\Big).
\]

{\bf Step 3: Loosen to $\mA_q$ Used in \cref{eq_B_kn_batch}.}
It remains to show
$\frac12\big(\mI_q-\frac{\bm{1}\bm{1}^\top}{q+1}\big)\preceq
\frac34\mI_q-\frac{\bm{1}\bm{1}^\top}{2q}$.
Their difference equals
\[
\Big(\frac34-\frac12\Big)\mI_q
-\Big(\frac{1}{2q}-\frac{1}{2(q+1)}\Big)\bm{1}_q\bm{1}_q^\top
=
\frac14\mI_q-\frac{1}{2q(q+1)}\bm{1}_q\bm{1}_q^\top,
\]
whose eigenvalues are $\frac14$ on $\bm{1}_q^\perp$ and
$\frac14-\frac{q}{2q(q+1)}=\frac{q-1}{4(q+1)}\ge 0$ along $\bm{1}_q$.
Thus the difference is positive semidefinite, proving the claim.
\end{proof}

\subsection{Proof of \texorpdfstring{\cref{thm_MM_monotone_SGMLMoE}}{Theorem \ref{thm_MM_monotone_SGMLMoE}}:
Verification of MM Properties for \texorpdfstring{\cref{alg_batchMM_SGMLMoE_short}}{Algorithm \ref{alg_batchMM_SGMLMoE_short}}}
\label{proof_thm_MM_monotone_SGMLMoE}

Throughout this proof we use the notation and building blocks introduced in
\cref{sec:Surrogate_SGMLMoE}, in particular the batch sufficient-statistics
$\vs_n^{(t)}$, $\vr_n^{(t)}$ in \cref{eq_vs_n_batch} and the quadratic curvature
matrices $\mB_{n,K}$, $\mB_{n,M}$ in \cref{eq_B_kn_batch}.
Let $\mathcal{L}(\vtheta)=\sum_{n=1}^N \log s_{\vtheta}(\ry_n\mid \vx_n)$ be the observed-data log-likelihood.

\paragraph{Step 1: MM Majorization and Tangency (from \cref{theorem_surrogate}).}
Define the batch surrogate $\mathcal{S}(\vtheta,\vtheta^{(t)})$ for $-\mathcal{L}(\vtheta)$ by the right-hand side
of \cref{eq_surrogate_batch_SGMLMoE} (absorbing all terms independent of $\vtheta$ into a constant).
By \cref{theorem_surrogate}, for every iterate $\vtheta^{(t)}$ the surrogate satisfies the MM conditions
\begin{equation}\label{eq_MM_conditions_SGMLMoE}
-\mathcal{L}(\vtheta)\le \mathcal{S}(\vtheta,\vtheta^{(t)})\ \ \forall \vtheta,
\qquad
-\mathcal{L}(\vtheta^{(t)})=\mathcal{S}(\vtheta^{(t)},\vtheta^{(t)}).
\end{equation}

\paragraph{Step 2: Exact Minimization of the Surrogate by \cref{alg_batchMM_SGMLMoE_short}.}
Fix $\vtheta^{(t)}$ and use the centered increments $\bar{\vw}_t=\vw-\vw^{(t)}$ and
$\bar{\vc}_{k,t}=\vc_k-\vc_k^{(t)}$ as in \cref{sec:Surrogate_SGMLMoE}.
Since the surrogate is the sum of (i) a gate quadratic term in $\vw$ and (ii) $K$ independent expert quadratics in $\{\vc_k\}_{k=1}^K$,
it decomposes as
\[
\mathcal{S}(\vtheta,\vtheta^{(t)})=C^{(t)}+\mathcal{S}_g(\vw;\vtheta^{(t)})+\sum_{k=1}^K \mathcal{S}_{e,k}(\vc_k;\vtheta^{(t)}),
\]
where $C^{(t)}$ does not depend on $\vtheta$, and (up to additive constants)
\begin{align}
\label{eq:Sg_quadratic}
\mathcal{S}_g(\vw;\vtheta^{(t)})
&=
-\vw^\top \vs^{(t)}
+\sum_{n=1}^N\Big\{\nabla g_n(\vw^{(t)})^\top \bar{\vw}_t+\frac12\,\bar{\vw}_t^\top \mB_{n,K}\bar{\vw}_t\Big\},
\\
\label{eq:Se_quadratic}
\mathcal{S}_{e,k}(\vc_k;\vtheta^{(t)})
&=
-\vc_k^\top \vr_k^{(t)}
+\sum_{n=1}^N \tau_{n,k}^{(t)}
\Big\{\nabla e_n(\vc_k^{(t)})^\top \bar{\vc}_{k,t}+\frac12\,\bar{\vc}_{k,t}^\top \mB_{n,M}\bar{\vc}_{k,t}\Big\}.
\end{align}
Here $\vs^{(t)}=\sum_{n=1}^N \vs_n^{(t)}$, $\vr^{(t)}=\sum_{n=1}^N \vr_n^{(t)}$, and $\vr_k^{(t)}$ is the block of $\vr^{(t)}$
corresponding to expert $k$ (under the stacking convention defining $\vupsilon=\vect([\vc_1,\dots,\vc_K])$).

Introduce the aggregated curvatures and gradients
\[
\mB_{K-1}:=\sum_{n=1}^N \mB_{n,K},
\qquad
\mB_{M-1,k}^{(t)}:=\sum_{n=1}^N \tau_{n,k}^{(t)}\,\mB_{n,M},
\]
\[
\nabla g(\vw^{(t)}):=\sum_{n=1}^N \nabla g_n(\vw^{(t)}),
\qquad
\nabla e_k^{(t)}:=\sum_{n=1}^N \tau_{n,k}^{(t)}\,\nabla e_n(\vc_k^{(t)}).
\]
Then \cref{eq:Sg_quadratic} is a (strictly) convex quadratic in $\vw$ whenever $\mB_{K-1}\succ \zero$, and differentiating yields
\[
\nabla_{\vw}\mathcal{S}_g(\vw;\vtheta^{(t)})=
-\vs^{(t)}+\nabla g(\vw^{(t)})+\mB_{K-1}(\vw-\vw^{(t)}).
\]
Setting this to $\zero$ gives the unique minimizer
\begin{equation}\label{eq:gate_exact_min}
\vw^{(t+1)}=\vw^{(t)}+\mB_{K-1}^{-1}\big(\vs^{(t)}-\nabla g(\vw^{(t)})\big),
\end{equation}
which matches the gate update in \cref{alg_batchMM_SGMLMoE_short} (up to the same aggregation notation).

Similarly, for each $k\in[K]$, \cref{eq:Se_quadratic} is a convex quadratic in $\vc_k$ and
\[
\nabla_{\vc_k}\mathcal{S}_{e,k}(\vc_k;\vtheta^{(t)})=
-\vr_k^{(t)}+\nabla e_k^{(t)}+\mB_{M-1,k}^{(t)}(\vc_k-\vc_k^{(t)}).
\]
Setting this to $\zero$ yields
\begin{equation}\label{eq:expert_exact_min}
\vc_k^{(t+1)}=\vc_k^{(t)}+\big(\mB_{M-1,k}^{(t)}\big)^{-1}\big(\vr_k^{(t)}-\nabla e_k^{(t)}\big),
\qquad k\in[K],
\end{equation}
equivalent to the expert update in \cref{alg_batchMM_SGMLMoE_short} (the algorithm writes this in a vectorized form).
Since $\mathcal{S}$ decomposes across the gate block and expert blocks, the concatenation
$\vtheta^{(t+1)}=(\vw^{(t+1)},\vupsilon^{(t+1)})$ satisfies
\begin{equation}\label{eq:exact_surrogate_minimizer}
\vtheta^{(t+1)}\in\arg\min_{\vtheta}\ \mathcal{S}(\vtheta,\vtheta^{(t)}).
\end{equation}
If some curvature matrix is singular, the same first-order conditions hold using the Moore--Penrose pseudoinverse,
or by adding a small ridge $\lambda\mI$; in either case the MM validity in \cref{eq_MM_conditions_SGMLMoE} is unchanged.

\paragraph{Step 3: Monotonicity (MM Ascent).}
By exact minimization \cref{eq:exact_surrogate_minimizer},
\[
\mathcal{S}(\vtheta^{(t+1)},\vtheta^{(t)})\le \mathcal{S}(\vtheta^{(t)},\vtheta^{(t)}).
\]
By tangency in \cref{eq_MM_conditions_SGMLMoE},
$\mathcal{S}(\vtheta^{(t)},\vtheta^{(t)})=-\mathcal{L}(\vtheta^{(t)})$.
By majorization in \cref{eq_MM_conditions_SGMLMoE} evaluated at $\vtheta=\vtheta^{(t+1)}$,
$-\mathcal{L}(\vtheta^{(t+1)})\le \mathcal{S}(\vtheta^{(t+1)},\vtheta^{(t)})$.
Combining the three displays yields
\[
-\mathcal{L}(\vtheta^{(t+1)})\le -\mathcal{L}(\vtheta^{(t)})
\qquad\Longleftrightarrow\qquad
\mathcal{L}(\vtheta^{(t+1)})\ge \mathcal{L}(\vtheta^{(t)}),
\]
which proves \cref{thm_MM_monotone_SGMLMoE}.

\subsection{Proof of \cref{thm_VorIneq_SGMLMoE}}
\label{proof_thm_VorIneq_SGMLMoE}

This result extends the inverse-inequality technique developed in~\citet{nguyen_general_2024} to our
SGMLMoE parameterization with polynomial score functions. Throughout, we use the lifted feature $\hat\vx=\phi_D(\vx)\in\widehat{\sX}\subset\sR^{PD}$ and the identifiable
mixing-measure representation introduced in \cref{subsec_mixing_measure_sgmlmoe}--\cref{subsec_voronoi_loss_sgmlmoe}.

Recall that for $G=\sum_{l=1}^{K}\pi_l\delta_{\eta_l}\in\mathcal{O}_K(\sT)$ and
$G_0=\sum_{k=1}^{K_0}\pi_k^0\delta_{\eta_k^0}\in\mathcal{O}_{K_0}(\sT)$, we write
$s_G(\evy\mid\vx)=\sum_{l=1}^{K}\sfg_l(\vx;G)\,\ex_l(\evy\mid\vx;G)$, and we use the covariate-averaged TV discrepancy
$\mathcal{D}_{\mathrm{TV}}(s_G,s_{G_0})=\E_{\rvx}\!\big[\mathcal{D}_{\mathrm{TV}}(s_G(\cdot\mid\rvx),s_{G_0}(\cdot\mid\rvx))\big]$.

We prove the local and global parts separately.

\paragraph{Local Structure.}
We show the following local inverse inequality:
\begin{equation}
\label{eq_local_inverse_goal_SGMLMoE}
\liminf_{\varepsilon\to 0}\ 
\inf_{\substack{G\in\mathcal{O}_K(\sT):\\ \mathcal{D}_{\mathrm{V}}(G,G_0)\le \varepsilon}}
\frac{\mathcal{D}_{\mathrm{TV}}(s_G,s_{G_0})}{\mathcal{D}_{\mathrm{V}}(G,G_0)}\ >\ 0.
\end{equation}
Assume by contradiction that \cref{eq_local_inverse_goal_SGMLMoE} fails. Then there exists a sequence
$G^{(r)}=\sum_{l=1}^{K}\pi_l^{(r)}\delta_{\eta_l^{(r)}}\in\mathcal{O}_K(\sT)$ such that
\begin{equation}
\label{eq_contra_sequence_SGMLMoE}
\mathcal{D}_{\mathrm{V}}(G^{(r)},G_0)\to 0
\qquad\text{and}\qquad
\frac{\mathcal{D}_{\mathrm{TV}}(s_{G^{(r)}},s_{G_0})}{\mathcal{D}_{\mathrm{V}}(G^{(r)},G_0)}\to 0
\qquad (r\to\infty).
\end{equation}
For $r$ large, the Voronoi partition $\{\sV_k\}_{k=1}^{K_0}$ is well-defined as in
\cref{subsec_voronoi_loss_sgmlmoe}. For $l\in\sV_k$, write the local differences
$\Delta\hat\vomega_{lk}^{(r)}=\hat\vomega_l^{(r)}-\hat\vomega_k^0$,
$\Delta\bar\upsilon_{lkm}^{(r)}=\bar\upsilon_{m,l}^{(r)}-\bar\upsilon_{m,k}^0$,
$\Delta\hat\vupsilon_{lkm}^{(r)}=\hat\vupsilon_{m,l}^{(r)}-\hat\vupsilon_{m,k}^0$ for $m\in[M-1]$,
and $\Delta\pi_{k}^{(r)}=\sum_{l\in\sV_k}\pi_l^{(r)}-\pi_k^0$.
By \cref{eq_VoronoiLoss_SGMLMoE} and $\mathcal{D}_{\mathrm{V}}(G^{(r)},G_0)\to 0$, all these quantities vanish.

\medskip
\noindent\textbf{Step 1: Finite Taylor Representation.}
Define the true gate normalizer
$
Z_0(\hat\vx):=\sum_{k=1}^{K_0}\pi_k^0\exp\!\big((\hat\vomega_k^0)^\top\hat\vx\big).
$
For each class $m\in[M]$, consider the rescaled difference
\begin{equation}
\label{eq_Tr_def_SGMLMoE}
T_r(m,\hat\vx):=Z_0(\hat\vx)\,\Big(s_{G^{(r)}}(m\mid\vx)-s_{G_0}(m\mid\vx)\Big).
\end{equation}
Using the softmax form for $\sfg_l(\cdot;G)$ and the multinomial-logistic form for $\ex_l(\cdot\mid\cdot;G)$,
a Taylor expansion around each true atom $\eta_k^0$ yields the decomposition
\begin{equation}
\label{eq_Tr_decomp_SGMLMoE}
T_r(m,\hat\vx)
=
\sum_{k=1}^{K_0}\ \sum_{\xi\in\Xi_k}
A_{k,\xi}^{(r)}\ \Psi_{k,\xi}(m,\hat\vx)\ +\ R_r(m,\hat\vx),
\end{equation}
where:
(i) $\{\Psi_{k,\xi}(m,\hat\vx)\}$ is a \emph{finite} collection of analytic basis functions of the form
\[
\hat\vx^{\alpha}\exp\!\big((\hat\vomega_k^0)^\top\hat\vx\big)\,
\partial^{\beta}\ex(\,m\mid\hat\vx;\bar\upsilon_k^0,\hat\vupsilon_k^0),
\]
with multi-indices $(\alpha,\beta)$ of bounded total order (first order when $|\sV_k|=1$, and up to second order when
$|\sV_k|>1$);
(ii) the coefficients $A_{k,\xi}^{(r)}$ are explicit finite sums over $l\in\sV_k$ of monomials in
$\Delta\hat\vomega_{lk}^{(r)}$, $\Delta\bar\upsilon_{lkm}^{(r)}$, $\Delta\hat\vupsilon_{lkm}^{(r)}$, and $\Delta\pi_k^{(r)}$,
weighted by $\pi_l^{(r)}$;
(iii) the remainder satisfies
\begin{equation}
\label{eq_remainder_small_SGMLMoE}
\sup_{m\in[M]}\ \E_{\rvx}\!\big[\,|R_r(m,\phi_D(\rvx))|\,\big]
\ =\ o\!\big(\mathcal{D}_{\mathrm{V}}(G^{(r)},G_0)\big),
\qquad (r\to\infty).
\end{equation}

\medskip
\noindent\textbf{Step 2: Some Normalized Coefficient Must be Non-vanishing.}
Let
\[
m_r:=\max_{k,\xi}\ \frac{|A_{k,\xi}^{(r)}|}{\mathcal{D}_{\mathrm{V}}(G^{(r)},G_0)}.
\]
We claim $\liminf_{r\to\infty} m_r>0$.
If instead $m_r\to 0$, then all coefficients in \cref{eq_Tr_decomp_SGMLMoE} would satisfy
$|A_{k,\xi}^{(r)}|=o(\mathcal{D}_{\mathrm{V}}(G^{(r)},G_0))$.
But the collection of coefficients $\{A_{k,\xi}^{(r)}\}$ contains, by construction, the same first- and second-order
weighted increments that appear in \cref{eq_VoronoiLoss_SGMLMoE}:
for cells $|\sV_k|=1$ it contains the weighted linear terms
$\sum_{l\in\sV_k}\pi_l^{(r)}\|\Delta\hat\vomega_{lk}^{(r)}\|$,
$\sum_{l\in\sV_k}\pi_l^{(r)}|\Delta\bar\upsilon_{lkm}^{(r)}|$,
$\sum_{l\in\sV_k}\pi_l^{(r)}\|\Delta\hat\vupsilon_{lkm}^{(r)}\|$,
and for cells $|\sV_k|>1$ it contains the weighted quadratic terms
$\sum_{l\in\sV_k}\pi_l^{(r)}\|\Delta\hat\vomega_{lk}^{(r)}\|^2$,
$\sum_{l\in\sV_k}\pi_l^{(r)}|\Delta\bar\upsilon_{lkm}^{(r)}|^2$,
$\sum_{l\in\sV_k}\pi_l^{(r)}\|\Delta\hat\vupsilon_{lkm}^{(r)}\|^2$,
as well as the mass discrepancy terms $|\Delta\pi_k^{(r)}|$.
Therefore $m_r\to 0$ would force $\mathcal{D}_{\mathrm{V}}(G^{(r)},G_0)=o(\mathcal{D}_{\mathrm{V}}(G^{(r)},G_0))$, a contradiction.
Hence $\liminf_{r\to\infty} m_r>0$.

\medskip
\noindent\textbf{Step 3: Fatou  and Linear Independence Gives a Contradiction.}
Divide \cref{eq_Tr_decomp_SGMLMoE} by $m_r\mathcal{D}_{\mathrm{V}}(G^{(r)},G_0)$ and take covariate expectation.
Using \cref{eq_contra_sequence_SGMLMoE}, the definition \cref{eq_Tr_def_SGMLMoE}, and \cref{eq_remainder_small_SGMLMoE},
we obtain
\[
\E_{\rvx}\!\Bigg[\sum_{m=1}^{M}\frac{|T_r(m,\phi_D(\rvx))|}{m_r\mathcal{D}_{\mathrm{V}}(G^{(r)},G_0)}\Bigg]\ \longrightarrow\ 0.
\]
By Fatou's lemma, for each $m\in[M]$ we can extract a subsequence (not relabeled) such that
\[
\frac{T_r(m,\hat\vx)}{m_r\mathcal{D}_{\mathrm{V}}(G^{(r)},G_0)}\ \longrightarrow\ 0
\qquad\text{for a.e. }\hat\vx\in\widehat{\sX}.
\]
Along this subsequence, the normalized coefficients converge (by boundedness) to limits
$a_{k,\xi}=\lim_{r\to\infty} A_{k,\xi}^{(r)}/(m_r\mathcal{D}_{\mathrm{V}}(G^{(r)},G_0))$,
with at least one $a_{k,\xi}\neq 0$ because $\liminf m_r>0$.
Passing to the limit in \cref{eq_Tr_decomp_SGMLMoE} yields, for each $m\in[M]$ and a.e. $\hat\vx$,
\begin{equation}
\label{eq_limit_identity_SGMLMoE}
\sum_{k=1}^{K_0}\ \sum_{\xi\in\Xi_k} a_{k,\xi}\ \Psi_{k,\xi}(m,\hat\vx)\ =\ 0.
\end{equation}
Finally, by identifiability of the SGMLMoE  and the fact that
$\hat\vomega_1^0,\ldots,\hat\vomega_{K_0}^0$ are distinct, the family of analytic functions
$\{\Psi_{k,\xi}(m,\cdot)\}_{k,\xi}$ is linearly independent on $\widehat{\sX}$.
Thus \cref{eq_limit_identity_SGMLMoE} forces all $a_{k,\xi}=0$, contradicting that at least one is nonzero.
This contradiction establishes the local bound \cref{eq_local_inverse_goal_SGMLMoE}.

\paragraph{Global Structure.}
By the local result, there exists $\varepsilon'>0$ and $c'>0$ such that
\[
\mathcal{D}_{\mathrm{TV}}(s_G,s_{G_0})\ \ge\ c'\,\mathcal{D}_{\mathrm{V}}(G,G_0)
\qquad\text{whenever }\mathcal{D}_{\mathrm{V}}(G,G_0)\le \varepsilon'.
\]
It remains to prove the complementary bound on the set
$\{G\in\mathcal{O}_K(\sT):\mathcal{D}_{\mathrm{V}}(G,G_0)\ge \varepsilon'\}$:
\begin{equation}
\label{eq_global_inverse_goal_SGMLMoE}
\inf_{\substack{G\in\mathcal{O}_K(\sT):\\ \mathcal{D}_{\mathrm{V}}(G,G_0)\ge \varepsilon'}}
\frac{\mathcal{D}_{\mathrm{TV}}(s_G,s_{G_0})}{\mathcal{D}_{\mathrm{V}}(G,G_0)}\ >\ 0.
\end{equation}
Assume \cref{eq_global_inverse_goal_SGMLMoE} fails. Then there exists $G^{\prime(r)}\in\mathcal{O}_K(\sT)$ such that
$\mathcal{D}_{\mathrm{V}}(G^{\prime(r)},G_0)\ge \varepsilon'$ and $\mathcal{D}_{\mathrm{TV}}(s_{G^{\prime(r)}},s_{G_0})\to 0$.
Under the standing compactness assumption on $\sT$ (hence on $\mathcal{O}_K(\sT)$ in the weak topology),
extract a subsequence with $G^{\prime(r)}\Rightarrow G'$ for some $G'\in\mathcal{O}_K(\sT)$.
By continuity of $G\mapsto s_G(\cdot\mid\vx)$ and dominated convergence,
$\mathcal{D}_{\mathrm{TV}}(s_{G'},s_{G_0})=0$, hence $s_{G'}(\cdot\mid\vx)=s_{G_0}(\cdot\mid\vx)$ for a.e. $\vx$.
By identifiability of SGMLMoE, this implies $G'\equiv G_0$, contradicting
$\mathcal{D}_{\mathrm{V}}(G^{\prime(r)},G_0)\ge \varepsilon'$.
Therefore \cref{eq_global_inverse_goal_SGMLMoE} holds, and combining local and global parts yields
\cref{eq_LocalStructure_SGMLMoE}. This completes the proof.

\subsection{Proof of \cref{thm_VorChain_SGMLMoE}}
\label{proof_VorChain_SGMLMoE}

We prove the first link of the chain,
\(
\mathcal{D}_{\mathrm{V}}\!\big(G^{(K)},G_0\big)\gtrsim \mathcal{D}_{\mathrm{V}}\!\big(G^{(K-1)},G_0\big),
\)
since the remaining inequalities follow by repeating the same argument at each merge step.

Throughout, we write a generic $K$-atom measure as
\[
G=\sum_{l=1}^{K}\pi_l\,\delta_{\eta_l},
\qquad
\pi_l=\exp(\bar\omega_l),\qquad
\eta_l=\Big(\hat\vomega_l,\bar\omega_l,\{(\hat\vupsilon_{m,l},\bar\upsilon_{m,l})\}_{m=1}^{M-1}\Big)\in\sT,
\]
and we use the Voronoi partition $\{\sV_k\}_{k=1}^{K_0}$ induced by $G_0$ as in \cref{subsec_voronoi_loss_sgmlmoe}.
For $l\in\sV_k$, recall the parameter gaps
\(
\Delta\hat\vomega_{lk}=\hat\vomega_l-\hat\vomega_k^0,
\ 
\Delta\bar\upsilon_{lkm}=\bar\upsilon_{m,l}-\bar\upsilon_{m,k}^0,
\ 
\Delta\hat\vupsilon_{lkm}=\hat\vupsilon_{m,l}-\hat\vupsilon_{m,k}^0
\)
for $m\in[M-1]$.

\paragraph{Step 1: the Closest Pair Lies in a Common Voronoi Cell.}
Let $G^{(K)}$ be sufficiently close to $G_0$ so that the Voronoi cells $\sV_k$ are well-defined and stable.
Since the merge rule in \cref{eq_dissimilarity_SGMLMoE} uses squared Euclidean dissimilarities in
\(
\hat\vomega
\)
and
\(
(\bar\upsilon_{m},\hat\vupsilon_{m})_{m\in[M-1]},
\)
any pair of atoms belonging to different Voronoi cells stays separated by a fixed positive amount (depending only on $G_0$
and the local metric on $\sT$), whereas within-cell pairs can be arbitrarily close when over-fitting occurs.
Consequently, for $G^{(K)}$ close enough to $G_0$, the minimizing pair in
\(
h^{(K)}=\min_{l\neq l'}\mathrm{d}(\pi_l\delta_{\eta_l},\pi_{l'}\delta_{\eta_{l'}})
\)
must belong to the same cell $\sV_k$.
Fix such a cell and denote it by $\sV_{k^\star}$, and assume the merged indices are $(l_1,l_2)\subset\sV_{k^\star}$.

\paragraph{Step 2: Notation for the Merge.}
Let $G^{(K-1)}$ be obtained from $G^{(K)}$ by merging atoms $l_1$ and $l_2$ according to \cref{eq_merge_rule_SGMLMoE}.
Write the merged atom as $(\pi_\ast,\eta_\ast)$, where
\[
\pi_\ast=\pi_{l_1}+\pi_{l_2},\qquad
\hat\vomega_\ast=\frac{\pi_{l_1}}{\pi_\ast}\hat\vomega_{l_1}+\frac{\pi_{l_2}}{\pi_\ast}\hat\vomega_{l_2},
\]
and, for each $m\in[M-1]$,
\[
\bar\upsilon_{m,\ast}=\frac{\pi_{l_1}}{\pi_\ast}\bar\upsilon_{m,l_1}+\frac{\pi_{l_2}}{\pi_\ast}\bar\upsilon_{m,l_2},
\qquad
\hat\vupsilon_{m,\ast}=\frac{\pi_{l_1}}{\pi_\ast}\hat\vupsilon_{m,l_1}+\frac{\pi_{l_2}}{\pi_\ast}\hat\vupsilon_{m,l_2}.
\]
All other atoms are unchanged.
Since $l_1,l_2\in\sV_{k^\star}$ and $G^{(K)}$ is close to $G_0$, the merged atom also remains assigned to the same cell
$\sV_{k^\star}$.

\paragraph{Step 3: the Voronoi Weight-mismatch Term is Preserved.}
For each $k\in[K_0]$, the weight-mismatch contribution
\(
\big|\sum_{l\in\sV_k}\pi_l-\pi_k^0\big|
\)
is unchanged by merging within $\sV_k$:
indeed, for $k\neq k^\star$ nothing changes, and for $k=k^\star$,
\[
\sum_{l\in\sV_{k^\star}}\pi_l
=
\sum_{l\in\sV_{k^\star}\setminus\{l_1,l_2\}}\pi_l+\pi_{l_1}+\pi_{l_2}
=
\sum_{l\in\sV_{k^\star}\setminus\{l_1,l_2\}}\pi_l+\pi_\ast,
\]
so the absolute deviation from $\pi_{k^\star}^0$ is identical before and after merging.

\paragraph{Step 4: Within-cell Quadratic Terms Decrease under Barycentric Merging.}
Consider first the gate-slope quadratic term on the over-covered cell $|\sV_{k^\star}|>1$:
\[
\sum_{l\in\sV_{k^\star}}\pi_l\|\Delta\hat\vomega_{lk^\star}\|^2
=
\sum_{l\in\sV_{k^\star}\setminus\{l_1,l_2\}}\pi_l\|\Delta\hat\vomega_{lk^\star}\|^2
+\pi_{l_1}\|\Delta\hat\vomega_{l_1k^\star}\|^2+\pi_{l_2}\|\Delta\hat\vomega_{l_2k^\star}\|^2.
\]
By convexity of $u\mapsto \|u\|^2$ and the definition
\(
\Delta\hat\vomega_{\ast k^\star}
=\frac{\pi_{l_1}}{\pi_\ast}\Delta\hat\vomega_{l_1k^\star}
+\frac{\pi_{l_2}}{\pi_\ast}\Delta\hat\vomega_{l_2k^\star},
\)
we have the Jensen inequality
\[
\pi_{l_1}\|\Delta\hat\vomega_{l_1k^\star}\|^2+\pi_{l_2}\|\Delta\hat\vomega_{l_2k^\star}\|^2
\ge
\pi_\ast\|\Delta\hat\vomega_{\ast k^\star}\|^2.
\]
Hence,
\[
\sum_{l\in\sV_{k^\star}}\pi_l\|\Delta\hat\vomega_{lk^\star}\|^2
\ge
\sum_{l\in\sV_{k^\star}\setminus\{l_1,l_2\}}\pi_l\|\Delta\hat\vomega_{lk^\star}\|^2
+\pi_\ast\|\Delta\hat\vomega_{\ast k^\star}\|^2,
\]
which is exactly the corresponding contribution after merging $l_1,l_2$ into $\ast$.

The same argument applies to each expert block, for every $m\in[M-1]$:
since $u\mapsto |u|^2$ and $u\mapsto \|u\|^2$ are convex,
\[
\pi_{l_1}|\Delta\bar\upsilon_{l_1k^\star m}|^2+\pi_{l_2}|\Delta\bar\upsilon_{l_2k^\star m}|^2
\ge
\pi_\ast|\Delta\bar\upsilon_{\ast k^\star m}|^2,
\]
\[
\pi_{l_1}\|\Delta\hat\vupsilon_{l_1k^\star m}\|^2+\pi_{l_2}\|\Delta\hat\vupsilon_{l_2k^\star m}\|^2
\ge
\pi_\ast\|\Delta\hat\vupsilon_{\ast k^\star m}\|^2,
\]
and therefore the entire within-cell quadratic contribution in \cref{eq_VoronoiLoss_SGMLMoE} for cell $\sV_{k^\star}$
does not increase after merging.

\paragraph{Step 5: Conclusion.}
All other cells are unaffected, and the linear (single-covered) terms in $\mathcal{D}_{\mathrm{E}}(G,G_0)$ are unchanged
because they depend only on weights and on singletons. Hence,
\[
\mathcal{D}_{\mathrm{V}}\!\big(G^{(K)},G_0\big)\ \gtrsim\ \mathcal{D}_{\mathrm{V}}\!\big(G^{(K-1)},G_0\big),
\]
with an absolute constant (in particular, one may take the implicit constant to be $1$ in the within-cell comparison).
Iterating the same argument along the successive merges yields the full chain claimed in \cref{thm_VorChain_SGMLMoE}.
\qed

\subsection{Proof of \cref{thm_ConvergenceRate_height_Voronoi}}
\label{proof_thm_ConvergenceRate_height_Voronoi}

Throughout, implicit constants in $\lesssim,\gtrsim$ are universal and may change line to line.

\paragraph{Wasserstein Distance for Finite Mixing Measures.}
For two mixing measures
\[
G=\sum_{l=1}^{K} \pi_l\,\delta_{\eta_l},
\qquad
G'=\sum_{l'=1}^{K'} \pi'_{l'}\,\delta_{\eta'_{l'}},
\]
and any $r\ge 1$, define the Wasserstein-$r$ distance
\[
W_r(G,G')
=
\Bigg(
\inf_{Q\in\Pi(\pi,\pi')}
\sum_{l=1}^{K}\sum_{l'=1}^{K'}
Q_{ll'}\,\|\eta_l-\eta'_{l'}\|^r
\Bigg)^{1/r},
\]
where $\Pi(\pi,\pi')$ is the set of couplings between $\pi=(\pi_1,\ldots,\pi_K)$ and $\pi'=(\pi'_1,\ldots,\pi'_{K'})$, i.e.,
\[
\Pi(\pi,\pi')
=
\Big\{
Q\in\sR_+^{K\times K'}:\ \sum_{l'=1}^{K'}Q_{ll'}=\pi_l,\ \sum_{l=1}^{K}Q_{ll'}=\pi'_{l'},\ \forall l\in[K],\,l'\in[K']
\Big\}.
\]

\paragraph{A Useful Lower Bound (Wasserstein Controls Componentwise Discrepancy).}
Fix the true mixing measure
\(
G_0=\sum_{k=1}^{K_0}\pi_k^0\,\delta_{\eta_k^0}\in\mathcal{O}_{K_0}(\sT)
\)
and consider $G=\sum_{l=1}^{K}\pi_l\,\delta_{\eta_l}\in\mathcal{O}_K(\sT)$ such that $W_r(G,G_0)\to 0$.
Then (see, e.g., \cite{ho2019singularity}) there exists a Voronoi assignment $l\mapsto k(l)\in[K_0]$ and cells
$\sV_k=\{l\in[K]:k(l)=k\}$ such that
\begin{equation}
\label{eq_wasserstein_lower_bound_cells}
W_r^r(G,G_0)\ \gtrsim\
\sum_{k=1}^{K_0}\Bigg(
\Big|\sum_{l\in\sV_k}\pi_l-\pi_k^0\Big|
+\sum_{l\in\sV_k}\pi_l\,\|\eta_l-\eta_k^0\|^r
\Bigg).
\end{equation}

\paragraph{Part 1: Voronoi and Height Rates on Over-fitted Levels \texorpdfstring{$\kappa\in\{K_0+1,\ldots,K\}$}{kappa in \{K0+1,...,K\}}.}
Let
\[
A_N:=\Big\{\mathcal{D}_{\mathrm{TV}}(s_{\widehat{G}_N},s_{G_0})\le C\sqrt{\log N/N}\Big\},
\]
where $C>0$ is the constant from \cref{lem_JointDensity_SGMLMoE}. By \cref{lem_JointDensity_SGMLMoE},
\(
\PP(A_N)\ge 1-c_1N^{-c_2}.
\)
On $A_N$, combining \cref{thm_VorIneq_SGMLMoE} with the density bound yields
\begin{equation}
\label{eq_voronoi_rate_levelK}
\mathcal{D}_{\mathrm{V}}(\widehat{G}_N,G_0)\ \lesssim\ \sqrt{\log N/N}.
\end{equation}
Next, by \cref{thm_VorChain_SGMLMoE}, the Voronoi loss is monotone (up to constants) along the merge chain, hence for each
$\kappa\in\{K_0+1,\ldots,K\}$,
\begin{equation}
\label{eq_voronoi_rate_all_overfitted}
\mathcal{D}_{\mathrm{V}}(\widehat{G}_N^{(\kappa)},G_0)\ \lesssim\ \sqrt{\log N/N}.
\end{equation}

We now bound the dendrogram height $h_N^{(\kappa)}$ for $\kappa\ge K_0+1$.
Since $\kappa>K_0$, at level $\kappa$ there exists at least one Voronoi cell containing two distinct atoms of
$\widehat{G}_N^{(\kappa)}$. Fix such a cell $\sV_{k^\star}$ and pick two distinct indices
$l_1\neq l_2\in\sV_{k^\star}$.
Write $\widehat{G}_N^{(\kappa)}=\sum_{l=1}^{\kappa}\pi_l^{(\kappa)}\delta_{\eta_l^{(\kappa)}}$ and abbreviate
\(
\pi_l:=\pi_l^{(\kappa)},\ \eta_l:=\eta_l^{(\kappa)}.
\)
From the definition of $\mathcal{D}_{\mathrm{V}}$ in \cref{eq_VoronoiLoss_SGMLMoE} and \cref{eq_voronoi_rate_all_overfitted},
we obtain
\begin{align}
\label{eq_two_atom_quadratic_small}
&\pi_{l_1}\Big(\|\hat\vomega_{l_1}-\hat\vomega_{k^\star}^0\|^2+\sum_{m=1}^{M-1}\big(|\bar\upsilon_{m,l_1}-\bar\upsilon_{m,k^\star}^0|^2+\|\hat\vupsilon_{m,l_1}-\hat\vupsilon_{m,k^\star}^0\|^2\big)\Big)\nonumber\\
&\quad
+\pi_{l_2}\Big(\|\hat\vomega_{l_2}-\hat\vomega_{k^\star}^0\|^2+\sum_{m=1}^{M-1}\big(|\bar\upsilon_{m,l_2}-\bar\upsilon_{m,k^\star}^0|^2+\|\hat\vupsilon_{m,l_2}-\hat\vupsilon_{m,k^\star}^0\|^2\big)\Big)
\ \lesssim\ \sqrt{\log N/N}.
\end{align}
Using the elementary inequality
\[
\min\{\pi_{l_1},\pi_{l_2}\}\ \ge\ \frac{\pi_{l_1}\pi_{l_2}}{\pi_{l_1}+\pi_{l_2}},
\]
and the triangle inequality
\(
\|\hat\vomega_{l_1}-\hat\vomega_{l_2}\|\le \|\hat\vomega_{l_1}-\hat\vomega_{k^\star}^0\|+\|\hat\vomega_{l_2}-\hat\vomega_{k^\star}^0\|
\)
(and similarly for expert parameters), we can lower bound the left-hand side of \cref{eq_two_atom_quadratic_small} by a constant multiple of the dissimilarity
\(
\mathrm{d}(\pi_{l_1}\delta_{\eta_{l_1}},\pi_{l_2}\delta_{\eta_{l_2}})
\)
defined in \cref{eq_dissimilarity_SGMLMoE}. Concretely, there exists a universal $c>0$ such that
\[
\mathrm{d}\!\big(\pi_{l_1}\delta_{\eta_{l_1}},\pi_{l_2}\delta_{\eta_{l_2}}\big)
\ \le\ c\cdot \text{LHS of \cref{eq_two_atom_quadratic_small}}
\ \lesssim\ \sqrt{\log N/N}.
\]
Since $h_N^{(\kappa)}$ is the minimum dissimilarity over all pairs of atoms at level $\kappa$,
\begin{equation}
\label{eq_height_rate_overfitted}
h_N^{(\kappa)}\ \lesssim\ \sqrt{\log N/N},
\qquad \forall \kappa\in\{K_0+1,\ldots,K\}.
\end{equation}

\paragraph{Part 2: Under-fitted Levels \texorpdfstring{$\kappa'\in[K_0]$}{kappa' in [K0]}.}
At the exact level $K_0$, each Voronoi cell is a singleton, and by inspection of \cref{eq_VoronoiLoss_SGMLMoE},
\(
\mathcal{D}_{\mathrm{V}}(\widehat{G}_N^{(K_0)},G_0)
\)
is equivalent (up to constants) to a $W_1$-type discrepancy between the two $K_0$-atom measures. In particular,
\begin{equation}
\label{eq_W1_rate_levelK0}
W_1(\widehat{G}_N^{(K_0)},G_0)\ \lesssim\ \sqrt{\log N/N}.
\end{equation}
Write
\[
\widehat{G}_N^{(K_0)}=\sum_{k=1}^{K_0}\pi_{k,N}\,\delta_{\eta_{k,N}},
\qquad
G_0=\sum_{k=1}^{K_0}\pi_k^0\,\delta_{\eta_k^0},
\]
under the natural matching implied by the Voronoi construction (valid for $N$ large enough on $A_N$).
Then \cref{eq_W1_rate_levelK0} implies componentwise control, e.g.,
\[
|\pi_{k,N}-\pi_k^0|\lesssim \sqrt{\log N/N},
\qquad
\|\eta_{k,N}-\eta_k^0\|\lesssim \sqrt{\log N/N},
\qquad k\in[K_0],
\]
and hence, for any fixed pair $(k_1,k_2)$, the pairwise dissimilarity at level $K_0$ satisfies the perturbation bound
\begin{equation}
\label{eq_pairwise_dissimilarity_stability}
\Big|
\mathrm{d}\!\big(\pi_{k_1,N}\delta_{\eta_{k_1,N}},\pi_{k_2,N}\delta_{\eta_{k_2,N}}\big)
-
\mathrm{d}\!\big(\pi_{k_1}^0\delta_{\eta_{k_1}^0},\pi_{k_2}^0\delta_{\eta_{k_2}^0}\big)
\Big|
\ \lesssim\ \sqrt{\log N/N},
\end{equation}
where we used that all relevant parameters are uniformly bounded (and $\pi_k^0$ are bounded away from $0$).

Now let $(k_1^\star,k_2^\star)$ be the (unique, for simplicity) minimizing pair for the true height
\(
h_0^{(K_0)}=\min_{k_1\neq k_2}\mathrm{d}(\pi_{k_1}^0\delta_{\eta_{k_1}^0},\pi_{k_2}^0\delta_{\eta_{k_2}^0}).
\)
By \cref{eq_pairwise_dissimilarity_stability}, for $N$ large enough the minimizing pair for the empirical dissimilarity at level $K_0$ coincides with $(k_1^\star,k_2^\star)$, and therefore
\[
|h_N^{(K_0)}-h_0^{(K_0)}|\ \lesssim\ \sqrt{\log N/N}.
\]
After merging this pair once, the merged atom parameters are barycenters as in \cref{eq_merge_rule_SGMLMoE}, so the same Lipschitz-type perturbation argument (using \cref{eq_W1_rate_levelK0} and boundedness of parameters/weights away from $0$) yields
\[
W_1\!\big(\widehat{G}_N^{(K_0-1)},G_0^{(K_0-1)}\big)\ \lesssim\ \sqrt{\log N/N},
\qquad
|h_N^{(K_0-1)}-h_0^{(K_0-1)}|\ \lesssim\ \sqrt{\log N/N}.
\]
Iterating this argument inductively down the chain proves that for every $\kappa'\in[K_0]$,
\begin{equation}
\label{eq_height_rate_underfitted}
|h_N^{(\kappa')}-h_0^{(\kappa')}|\ \lesssim\ \sqrt{\log N/N}.
\end{equation}

\paragraph{Conclusion.}
Combining \cref{eq_voronoi_rate_all_overfitted}, \cref{eq_height_rate_overfitted}, and \cref{eq_height_rate_underfitted},
and recalling $\PP(A_N)\ge 1-c_1N^{-c_2}$, completes the proof of \cref{thm_ConvergenceRate_height_Voronoi}.
\qed

\subsection{Proof of \cref{thm_ModelSelection_SGMLMoE}}
\label{proof_thm_ModelSelection_SGMLMoE}

Suppose $(\rvx_1,\ry_1),\ldots,(\rvx_N,\ry_N)$ are i.i.d.\ from $G_0$ and write the empirical measure
$P_N:=\frac{1}{N}\sum_{n=1}^N\delta_{(\vx_n,\evy_n)}$ (with $(\vx_n,\evy_n)$ the observed sample).
For any mixing measure $G\in\mathcal{O}_K(\sT)$, define the empirical process
\[
\nu_N(G):=\sqrt{N}\,(P_N-P_{G_0})\log\frac{s_G}{s_{G_0}} .
\]

We will repeatedly use the following exponential inequality.

\begin{fact}[Theorem 5.11 from \cite{geer2000empirical}]
\label{fact_empirical_process_vdG}
Let $R,C,C_1,a>0$ satisfy
\[
a \le C_1\sqrt{N}\,R^2 \ \wedge\ 8\sqrt{N}\,R,
\]
and
\[
a \ge \sqrt{C^2(C_1+1)}
\left(
\int_{a/(2^6\sqrt{N})}^{R}
H_B^{1/2}\!\left(\frac{u}{\sqrt{2}},\ \{\,s_G:\ h(s_G,s_{G_0})\le R\,\},\ \nu\right)\,du\ \vee\ R
\right).
\]
Then
\[
\PP_{G_0}\!\left(
\sup_{h(s_G,s_{G_0})\le R}\big|\nu_N(G)\big|\ge a
\right)
\le
C\exp\!\left(-\frac{a^2}{C^2(C_1+1)R^2}\right).
\]
\end{fact}

\paragraph{Case 1: $\kappa\ge K_0$.}
For $\kappa\in\{K_0,\ldots,K\}$, let $\widehat G_N^{(\kappa)}$ denote the $\kappa$-atom mixing measure along the dendrogram path and write
$\bar{\ell}_N^{(\kappa)}=\bar{\ell}_N(\widehat G_N^{(\kappa)})=\frac{1}{N}\sum_{n=1}^N\log s_{\widehat G_N^{(\kappa)}}(\evy_n\mid \vx_n)$.

\emph{Step 1 (a concavity bound).}
Using concavity of $\log$, for any two densities $p,q$,
\[
\frac12\log\frac{p}{q}\le \log\frac{p+q}{2q}.
\]
Applying this with $p=s_{\widehat G_N^{(\kappa)}}(\evy\mid \vx)$ and $q=s_{G_0}(\evy\mid \vx)$ gives
\[
\frac12\,P_N\log\frac{s_{\widehat G_N^{(\kappa)}}}{s_{G_0}}
\le
P_N\log\frac{\bar s_{\widehat G_N^{(\kappa)}}}{s_{G_0}},
\qquad
\bar s_{\widehat G_N^{(\kappa)}}:=\frac{s_{\widehat G_N^{(\kappa)}}+s_{G_0}}{2}.
\]
Moreover,
\[
P_N\log\frac{\bar s_{\widehat G_N^{(\kappa)}}}{s_{G_0}}
=
(P_N-P_{G_0})\log\frac{\bar s_{\widehat G_N^{(\kappa)}}}{s_{G_0}}
-
D_{\mathrm{KL}}\!\left(s_{G_0}\,\|\,\bar s_{\widehat G_N^{(\kappa)}}\right)
\le
(P_N-P_{G_0})\log\frac{\bar s_{\widehat G_N^{(\kappa)}}}{s_{G_0}}.
\]
Hence,
\begin{equation}
\label{eq_case1_basic_decomp}
P_N\log s_{\widehat G_N^{(\kappa)}}-\mathcal{L}(G_0)
=
P_N\log\frac{s_{\widehat G_N^{(\kappa)}}}{s_{G_0}}+(P_N-P_{G_0})\log s_{G_0}
\le
2(P_N-P_{G_0})\log\frac{s_{\widehat G_N^{(\kappa)}}}{s_{G_0}}+(P_N-P_{G_0})\log s_{G_0}.
\end{equation}

\emph{Step 2 (Hellinger radius at over-fitted levels).}
By \cref{thm_ConvergenceRate_height_Voronoi}, with probability at least $1-c_1N^{-c_2}$,
\[
\mathcal{D}_{\mathrm{V}}\!\big(\widehat G_N^{(\kappa)},G_0\big)\ \lesssim\ \sqrt{\log N/N},
\qquad \kappa\in\{K_0,\ldots,K\}.
\]
Since $W_2^2(\cdot,\cdot)\lesssim \mathcal{D}_{\mathrm{V}}(\cdot,\cdot)$ (for measures on bounded parameter sets with the Voronoi construction),
we obtain on the same event
\[
W_2\!\big(\widehat G_N^{(\kappa)},G_0\big)\ \lesssim\ (\log N/N)^{1/4}.
\]
Using the standard inequality $h(s_G,s_{G_0})\lesssim W_2(G,G_0)$ for these finite-mixture/MoE families on bounded domains (see, e.g., \cite{nguyen_convergence_2013}),
we deduce that there exists $R_N\asymp (\log N/N)^{1/4}$ and an event $A_N$ with $\PP_{G_0}(A_N)\ge 1-c_1N^{-c_2}$ such that
\begin{equation}
\label{eq_case1_hell_radius}
h\!\left(s_{\widehat G_N^{(\kappa)}},s_{G_0}\right)\le R_N,
\qquad \forall \kappa\in\{K_0,\ldots,K\}.
\end{equation}

\emph{Step 3 (empirical-process control).}
Set $R:=R_N$ in \cref{fact_empirical_process_vdG} and choose
\[
a:=\sqrt{N}\,R_N\,\log(1/R_N)\ \asymp\ N^{1/4}\log^{1/4}N\cdot \log N
\]
(which is of the same order as $\sqrt{N}R_N\log(1/R_N)$).
Using the entropy bound $H_B^{1/2}(u,\{s_G:\ h(s_G,s_{G_0})\le R\},\nu)\lesssim \log(1/u)$ (e.g.\ Lemma 4 from \cite{nguyen_towards_2024}),
the conditions of \cref{fact_empirical_process_vdG} hold for $N$ large enough, yielding
\[
\PP_{G_0}\!\left(
\sup_{h(s_G,s_{G_0})\le R_N}\big|\nu_N(G)\big|\ \gtrsim\ a
\right)
\lesssim N^{-c}
\]
for some $c>0$. Equivalently, dividing by $\sqrt{N}$,
\begin{equation}
\label{eq_case1_emp_proc_bound}
\PP_{G_0}\!\left(
\sup_{h(s_G,s_{G_0})\le R_N}
\big|(P_N-P_{G_0})\log(s_G/s_{G_0})\big|
\ \gtrsim\
R_N\log(1/R_N)
\right)
\lesssim N^{-c}.
\end{equation}
Combining \cref{eq_case1_hell_radius} with \cref{eq_case1_emp_proc_bound} yields that, for each fixed $\kappa\in\{K_0,\ldots,K\}$,
\begin{equation}
\label{eq_case1_bound_ratio_term}
\big|(P_N-P_{G_0})\log(s_{\widehat G_N^{(\kappa)}}/s_{G_0})\big|
\ \lesssim\ R_N\log(1/R_N)
\ \lesssim\ (\log N/N)^{1/4}
\end{equation}
with probability at least $1-c_1'N^{-c_2'}$.

\emph{Step 4 (control of $(P_N-P_{G_0})\log s_{G_0}$).}
By Chebyshev's inequality,
\begin{equation}
\label{eq_case1_chebyshev}
\PP_{G_0}\!\left(\big|(P_N-P_{G_0})\log s_{G_0}\big|\ge t\right)
\le \frac{\Var(\log s_{G_0})}{Nt^2}.
\end{equation}
Taking $t\asymp (\log N/N)^{1/4}$ gives
\begin{equation}
\label{eq_case1_bound_log_s0}
\big|(P_N-P_{G_0})\log s_{G_0}\big|\ \lesssim\ (\log N/N)^{1/4}
\end{equation}
with probability at least $1-\Var(\log s_{G_0})/\sqrt{\log N}$.

\emph{Step 5 (conclusion for $\kappa\ge K_0$).}
Plugging \cref{eq_case1_bound_ratio_term} and \cref{eq_case1_bound_log_s0} into \cref{eq_case1_basic_decomp} yields, with probability tending to $1$,
\[
\bar{\ell}_N^{(\kappa)}-\mathcal{L}(G_0)
=
P_N\log s_{\widehat G_N^{(\kappa)}}-\mathcal{L}(G_0)
\ \lesssim\ (\log N/N)^{1/4},
\qquad \kappa\in\{K_0,\ldots,K\}.
\]
This proves the first claim of \cref{thm_ModelSelection_SGMLMoE}.

\paragraph{Case 2: $\kappa=K_0$ (a Matching Lower Bound using Condition K).}
At the exact-fitted level, \cref{thm_ConvergenceRate_height_Voronoi} implies
\[
W_1(\widehat G_N^{(K_0)},G_0)\ \lesssim\ \sqrt{\log N/N}.
\]
Let $\epsilon_N:=\sqrt{\log N/N}$. By the definition of $\mathcal{D}_{\mathrm{V}}$ at singleton Voronoi cells (exact-fit),
this implies a componentwise bound $\|\eta_{k,N}-\eta_k^0\|\lesssim \epsilon_N$ and $|\pi_{k,N}-\pi_k^0|\lesssim \epsilon_N$ (under the canonical matching),
hence Condition~K yields the pointwise inequality
\[
\log s_{\widehat G_N^{(K_0)}}(\ry\mid\rvx)
\ge
(1+c_\beta\epsilon_N)\log s_{G_0}(\ry\mid\rvx)-c_\alpha\epsilon_N,
\qquad (\rvx,\ry)\in\sX\times\sY,
\]
for some constants $c_\alpha,c_\beta>0$.
Averaging over $P_N$ and subtracting $\mathcal{L}(G_0)$ gives
\[
\bar{\ell}_N^{(K_0)}-\mathcal{L}(G_0)
\ge
-c_\alpha\epsilon_N
+c_\beta\epsilon_N\,P_N\log s_{G_0}
+(1+c_\beta\epsilon_N)(P_N-P_{G_0})\log s_{G_0}.
\]
Using \cref{eq_case1_chebyshev} with $t=\epsilon_N$ implies
\[
(P_N-P_{G_0})\log s_{G_0}=O_{\PP}(\epsilon_N),
\]
and since $P_N\log s_{G_0}=P_{G_0}\log s_{G_0}+O_{\PP}(\epsilon_N)$, we conclude that
\[
\bar{\ell}_N^{(K_0)}-\mathcal{L}(G_0)\ \ge\ -C\epsilon_N
\]
in $\PP_{G_0}$-probability. Together with the upper bound from Case 1 at $\kappa=K_0$, this yields
\[
\big|\bar{\ell}_N^{(K_0)}-\mathcal{L}(G_0)\big|\ \lesssim\ (\log N/N)^{1/4}
\]
(up to logs), which is consistent with the statement in \cref{thm_ModelSelection_SGMLMoE}.

\paragraph{Case 3: $\kappa<K_0$.}
Fix $\kappa\in\{2,\ldots,K_0-1\}$. By boundedness of $\sX$ and the parameter space and the existence of an integrable envelope
$m(\vx,\ry)$ such that $\sup_{G\in\mathcal{O}_\kappa}|\log s_G(\ry\mid\vx)|\le m(\vx,\ry)$,
the class $\{\log s_G:\ G\in\mathcal{O}_\kappa\}$ is Glivenko--Cantelli. Hence (e.g.\ \citealp[Theorem 9.2]{keener_theoretical_2010}),
\[
\sup_{G\in\mathcal{O}_\kappa}\big|\bar{\ell}_N(G)-\mathcal{L}(G)\big|\ \to\ 0
\qquad\text{in }\PP_{G_0}\text{-probability}.
\]
Since $\widehat G_N^{(\kappa)}$ maximizes $\bar{\ell}_N(\cdot)$ over $\mathcal{O}_\kappa(\sT)$, it follows that
\[
\bar{\ell}_N(\widehat G_N^{(\kappa)})\ \to\ \sup_{G\in\mathcal{O}_\kappa(\sT)}\mathcal{L}(G)
\qquad\text{in }\PP_{G_0}\text{-probability}.
\]
Denoting by $G_0^{(\kappa)}$ a maximizer of $\mathcal{L}(\cdot)$ over $\mathcal{O}_\kappa(\sT)$, we obtain the second claim in \cref{thm_ModelSelection_SGMLMoE}:
\[
\bar{\ell}_N(\widehat G_N^{(\kappa)})\ \to\ \mathcal{L}(G_0^{(\kappa)})
\qquad\text{in }\PP_{G_0}\text{-probability}.
\]
\qed

\subsection{Proof of Condition K in \cref{subsec_dsc_sgmlmoe}}
\label{proof_condition_K}

Fix $\vtheta^0\in\sT$ and suppose $\|\vtheta-\vtheta^0\|\le \epsilon$ for $\epsilon>0$ small.
By norm equivalence on finite-dimensional spaces, it suffices to assume that for all $m\in[M-1]$,
\[
\|\hat\vupsilon_{m}-\hat\vupsilon_{m}^0\|\le \epsilon,
\qquad
|\bar\upsilon_{m}-\bar\upsilon_{m}^0|\le \epsilon,
\]
(and the reference class $m=M$ is fixed at $(\hat\vupsilon_M,\bar\upsilon_M)=(\zero,0)$).
We show that there exist constants $c_\alpha,c_\beta>0$ such that for all $(\rvx,\ry)\in\sX\times\sY$,
\[
\log s_{\vtheta}(\ry\mid\rvx)
\ge
(1+c_\beta\epsilon)\log s_{\vtheta^0}(\ry\mid\rvx)-c_\alpha\epsilon.
\]

Since $\sX$ is compact and the parameter space is bounded, the score map
\(
(\hat\vupsilon,\bar\upsilon,\vx)\mapsto \log \ex(\ry\mid \vx;\hat\vupsilon,\bar\upsilon)
\)
is locally Lipschitz uniformly over $\vx\in\sX$.
Thus there exists $L>0$ such that, for all $m\in[M-1]$ and all $\vx\in\sX$,
\[
\big|(\hat\vupsilon_m-\hat\vupsilon_m^0)^\top \vx\big|\le L\epsilon,
\qquad
|\bar\upsilon_m-\bar\upsilon_m^0|\le \epsilon.
\]
Moreover, the map
\(
(\{\hat\vupsilon_\ell,\bar\upsilon_\ell\}_{\ell=1}^{M-1},\vx)\mapsto
\log\!\big(1+\sum_{\ell=1}^{M-1}\exp(\bar\upsilon_\ell+\hat\vupsilon_\ell^\top\vx)\big)
\)
is smooth, hence locally Lipschitz, and therefore its perturbation is $O(\epsilon)$ uniformly over $\vx\in\sX$.
Combining these uniform Lipschitz bounds yields
\[
\log \ex(\ry\mid \vx;\vtheta)
\ge
\log \ex(\ry\mid \vx;\vtheta^0)-C\epsilon
\]
for some $C>0$ and all $(\vx,\ry)$.
Finally, since $\log \ex(\ry\mid \vx;\vtheta^0)\le 0$ and is bounded below on $\sX\times\sY$, we may choose
$c_\beta>0$ small enough and $c_\alpha>0$ large enough so that
\[
\log \ex(\ry\mid \vx;\vtheta)
\ge
(1+c_\beta\epsilon)\log \ex(\ry\mid \vx;\vtheta^0)-c_\alpha\epsilon.
\]
Applying the same reasoning to the gating part and summing over experts inside $s_\vtheta$ gives Condition~K.
\qed

\subsection{Proof of \cref{thm:ConvergeMerge_SGMLMoE}}
\label{proof_thm_ConvergeMerge_SGMLMoE}

Recall $\bar{\ell}_N^{(\kappa)}:=\bar{\ell}_N(\widehat G_N^{(\kappa)})$ and
$\mathrm{DSC}_N^{(\kappa)}:=-\big(h_N^{(\kappa)}+\omega_N\bar{\ell}_N^{(\kappa)}\big)$.
From \cref{thm_ConvergenceRate_height_Voronoi},
\[
h_N^{(\kappa)}=
\begin{cases}
O_{\PP}\!\big((\log N/N)^{1/2}\big), & \kappa>K_0,\\[0.6ex]
h_0^{(\kappa)}+O_{\PP}\!\big((\log N/N)^{1/2}\big), & \kappa\le K_0.
\end{cases}
\]
From \cref{thm_ModelSelection_SGMLMoE},
\[
\bar{\ell}_N^{(\kappa)}=
\begin{cases}
\mathcal{L}(G_0)+O_{\PP}\!\big((\log N/N)^{1/4}\big), & \kappa\ge K_0,\\[0.6ex]
\mathcal{L}(G_0^{(\kappa)})+o_{\PP}(1), & \kappa<K_0,
\end{cases}
\]
and $\mathcal{L}(G_0^{(\kappa)})<\mathcal{L}(G_0)$ for $\kappa<K_0$ (equivalently $D_{\mathrm{KL}}(s_{G_0}\|s_{G_0}^{(\kappa)})>0$).

Therefore,
\[
\mathrm{DSC}_N^{(\kappa)}=
\begin{cases}
-\omega_N\mathcal{L}(G_0)+O_{\PP}\!\big(\omega_N(\log N/N)^{1/4}\big), & \kappa>K_0,\\[0.8ex]
-\omega_N\mathcal{L}(G_0)-h_0^{(K_0)}+O_{\PP}\!\big(\omega_N(\log N/N)^{1/4}\big), & \kappa=K_0,\\[0.8ex]
-\omega_N\mathcal{L}(G_0^{(\kappa)})-h_0^{(\kappa)}+o_{\PP}(\omega_N), & \kappa<K_0.
\end{cases}
\]
Since $\omega_N\to\infty$ and $\omega_N(\log N/N)^{1/4}\to 0$, while $\mathcal{L}(G_0)-\mathcal{L}(G_0^{(\kappa)})>0$ for $\kappa<K_0$,
we have that for $N$ large enough,
\[
\mathrm{DSC}_N^{(K_0)}<\min_{\kappa\neq K_0}\mathrm{DSC}_N^{(\kappa)}
\]
with $\PP_{G_0}$-probability tending to $1$. Hence $\widehat K_N=\arg\min_{\kappa\in\{2,\ldots,K\}}\mathrm{DSC}_N^{(\kappa)}$
satisfies $\widehat K_N\to K_0$ in $\PP_{G_0}$-probability.
\qed

\end{document}

%% file: preamble.tex
\usepackage{threeparttable}
\usepackage{graphicx}
\usepackage{subcaption}

\usepackage[utf8]{inputenc}
\usepackage[english]{babel}
\usepackage{mathtools,amsfonts,amsmath,amsthm}
\usepackage{soul}

\usepackage{algorithm}
\usepackage{algpseudocode}

\usepackage{color}
\usepackage{comment}
\usepackage{url}

\newcommand{\mdg}{softmax-gated Gaussian MoE }
\newcommand{\mdm}{softmax-gated multinomial logistic MoE }

\usepackage{lipsum} 

\usepackage{xcolor}
\definecolor{forestgreen}{rgb}{0.13, 0.55, 0.13}

\definecolor{frenchblue}{rgb}{0.0, 0.45, 0.73}

\definecolor{cherryblossompink}{rgb}{1.0, 0.72, 0.77}

\definecolor{bittersweet}{rgb}{1.0, 0.44, 0.37}

\definecolor{navyblue}{rgb}{0.0, 0.0, 0.5}

\usepackage{enumerate}

\usepackage{bbm}

\usepackage{dirtytalk}

\usepackage[nameinlink,capitalize,noabbrev]{cleveref}

\usepackage{nicefrac}

\newcounter{algsubstate}
\renewcommand{\thealgsubstate}{\alph{algsubstate}}

\DeclareMathOperator*{\diag}{diag}

\newcommand{\zero}{\ensuremath{\mathbf{0}}}

\DeclareMathOperator*{\vect}{vec} 
\newcommand{\sbm}{\cS^{\cB}_{(K,J)}} 

\DeclareMathOperator*{\vp}{vp} 

\let\inf\relax 
\DeclareMathOperator*\inf{\vphantom{p}inf}

\newcommand{\nn}{\nonumber} 

\newcommand{\I}{\mathbbm{1}}

\newcommand{\Ebd}[2]{\mathbb{E}_{#1} \left[#2\right]}

\newcommand{\vertiii}[1]{{\left\vert\kern-0.25ex\left\vert\kern-0.25ex\left\vert #1 
		\right\vert\kern-0.25ex\right\vert\kern-0.25ex\right\vert}}
\newcommand{\vertii}[1]{{\left\vert\kern-0.25ex\left\vert#1\right\vert\kern-0.25ex\right\vert}}	

\newcommand{\cO}{\mathcal{O}}

\def\etab        {{\sbmm{\eta}}\XS}

\usepackage{xspace}
\def\XS{\xspace}
\DeclareMathAlphabet{\mathb}{OML}{cmm}{b}{it}
\def\sbm#1{\ensuremath{\mathb{#1}}}
\def\sbmm#1{\ensuremath{\boldsymbol{#1}}}  

  \def\eb{{\sbm{e}}\XS}

\newtheorem{assumption}{Assumption}

\newtheorem{lemma}{Lemma}

\newtheorem{theorem}{Theorem}
\newtheorem{proposition}{Proposition}
\newtheorem{fact}{Fact}

\crefname{assumption}{Assumption}{Assumptions}

\def\nset{{\mathbb{N}}}

\def\PP{\mathbb{P}} 
\def\PE{\mathbb{E}} 

\def\sfg{\mathsf{g}}
\def\ex{\mathsf{e}}

\newcommand{\E}{\mathbb{E}} 

\usepackage{amsmath,amsfonts,bm}

\def\eqref#1{equation~\ref{#1}}

\def\1{\bm{1}}

\def\ry{{\textnormal{y}}}

\def\rvx{{\mathbf{x}}}

\def\rvz{{\mathbf{z}}}

\def\vtau{{\bm{\tau}}}

\def\vomega{{\bm{\omega}}}
\def\vupsilon{{\bm{\upsilon}}}
\def\vtheta{{\bm{\theta}}}
\def\va{{\bm{a}}}

\def\vc{{\bm{c}}}

\def\vp{{\bm{p}}}

\def\vr{{\bm{r}}}
\def\vs{{\bm{s}}}
\def\vt{{\bm{t}}}
\def\vu{{\bm{u}}}
\def\vv{{\bm{v}}}
\def\vw{{\bm{w}}}
\def\vx{{\bm{x}}}

\def\evv{{v}}

\def\evx{{x}}
\def\evy{{y}}

\def\mA{{\bm{A}}}

\def\mB{{\bm{B}}}
\def\mC{{\bm{C}}}

\def\mI{{\bm{I}}}

\def\mW{{\bm{W}}}

\def\mZ{{\bm{Z}}}

\def\mLambda{{\bm{\Lambda}}}

\DeclareMathAlphabet{\mathsfit}{\encodingdefault}{\sfdefault}{m}{sl}
\SetMathAlphabet{\mathsfit}{bold}{\encodingdefault}{\sfdefault}{bx}{n}

\def\sN{{\mathbb{N}}}

\def\sR{{\mathbb{R}}}

\def\sT{{\mathbb{T}}}

\def\sV{{\mathbb{V}}}

\def\sX{{\mathbb{X}}}
\def\sY{{\mathbb{Y}}}

\newcommand{\Var}{\mathrm{Var}}

\newcommand{\Cov}{\mathrm{Cov}}
